\documentclass{article}

\usepackage{microtype}
\usepackage{graphicx}
\usepackage{subcaption}
\usepackage{booktabs} 
\usepackage{enumitem}

\usepackage{hyperref}

\usepackage{xcolor}
\definecolor{DeepGreen}{RGB}{0,100,0} 
\DeclareRobustCommand{\algHL}[2]{%
  \begingroup
  \setlength{\fboxsep}{0.8pt}%
  \colorbox{#1}{\textcolor{black}{#2}}%
  \endgroup
}
\DeclareRobustCommand{\algLHM}[1]{\algHL{red!18}{#1}}      
\DeclareRobustCommand{\algYZ}[1]{\algHL{magenta!22}{#1}}    
\DeclareRobustCommand{\algCP}[1]{\algHL{green!25}{#1}}    

\DeclareRobustCommand{\algHLwrap}[2]{%
  \begingroup
  \setlength{\fboxsep}{0.8pt}%
  \colorbox{#1}{%
    \parbox[t]{\dimexpr\linewidth-2\fboxsep\relax}{\raggedright\textcolor{black}{#2}}%
  }%
  \endgroup
}

\DeclareRobustCommand{\algYZw}[1]{\algHLwrap{magenta!22}{#1}}

\usepackage[preprint]{icml2026}

\usepackage{amsmath}
\usepackage{amssymb}
\usepackage{mathtools}
\usepackage{amsthm}
\graphicspath{{figures/}}

\usepackage[capitalize,noabbrev]{cleveref}

\theoremstyle{plain}
\newtheorem{theorem}{Theorem}[section]

\newtheorem{lemma}[theorem]{Lemma}

\theoremstyle{definition}

\newtheorem{assumption}[theorem]{Assumption}
\theoremstyle{remark}

\usepackage[textsize=tiny]{todonotes}

\icmltitlerunning{TrasMuon: Trust-Region Adaptive Scaling for Orthogonalized Momentum Optimizers}

\begin{document}

\twocolumn[
  \icmltitle{TrasMuon: Trust-Region Adaptive Scaling for Orthogonalized Momentum Optimizers}




  \icmlsetsymbol{equal}{*}

  \begin{icmlauthorlist}
    \icmlauthor{Peng Cheng}{equal,aaa}
    \icmlauthor{Jiucheng Zang}{sch}
    \icmlauthor{Qingnan Li}{aaa}
    \icmlauthor{Liheng Ma}{mila}
    \icmlauthor{Jimmy Jian}{aaa}
    \icmlauthor{Boxing Chen}{aaa}
    \icmlauthor{Yingxue Zhang}{aaa}    
    \icmlauthor{Yufei Cui}{aaa}
    \icmlauthor{Wen Tong}{aaa}
  \end{icmlauthorlist}

  \icmlaffiliation{sch}{Department of Combinatorics and Optimization, University of Waterloo, Waterloo, Canada}
  \icmlaffiliation{aaa}{Huawei Canadian Research Institute, Canada}
  \icmlaffiliation{mila}{McGill University \& Mila-Quebec AI Institue, Canada}
\icmlcorrespondingauthor{Peng Cheng}{peng.cheng.hit@gmail.com}

  \icmlkeywords{Optimizer, Muon}

  \vskip 0.3in
]



\printAffiliationsAndNotice{}  

\begin{abstract}
Muon-style optimizers leverage Newton-Schulz (NS) iterations to orthogonalize updates, yielding update geometries that often outperform Adam-series methods. However, this orthogonalization discards magnitude information, rendering training sensitive to step-size hyperparameters and vulnerable to high-energy bursts. To mitigate this, we introduce TrasMuon (\textbf{T}rust \textbf{R}egion \textbf{A}daptive \textbf{S}caling \textbf{Muon}). TrasMuon preserves the near-isometric geometry of Muon while stabilizing magnitudes through (i) global RMS calibration and (ii) energy-based trust-region clipping. We demonstrate that while reintroducing adaptive scaling improves optimization efficiency, it typically exacerbates instability due to high-energy outliers. TrasMuon addresses this by defining a trust region based on relative energy ratios, confining updates to a stable zone. Empirical experiments on vision and language models demonstrate that TrasMuon converges faster than baselines. Furthermore, experiments without warmup stages confirm TrasMuon’s superior stability and robustness.

\end{abstract}

\section{Introduction}

Optimizer choice remains a major bottleneck for training modern foundation models, affecting convergence speed, stability, and overall compute cost at scale~\cite{deepseek-aiDeepSeekV3TechnicalReport2025, IntroducingGPT522026, IntroducingClaudeOpus}.
Despite substantial progress in large-scale training recipes, practitioners still face highly heterogeneous gradient landscapes and heavy-tailed/outlier updates that can trigger loss spikes and narrow the stable learning-rate region~\cite{behrouzATLASLearningOptimally2025, teamKimiK2Open2025a, parkOutlierSafePreTrainingRobust2025}.
A dominant baseline family is diagonal adaptivity, including Adam and its variants~\cite{kingma2017adammethodstochasticoptimization,loshchilov2019decoupledweightdecayregularization, pagliardini2024ademamixoptimizerbetterfaster,marfinetzEvolvingDeepLearning2025}, which provides robust coordinate-wise magnitude control but does not explicitly leverage matrix-level update structure.

Recent work on momentum orthogonalization has revived interest in matrix-structured updates for Transformers~\cite{bernstein2025deriving,jordan2024muon}.
Muon-style optimizers reshape update directions via finite-step Newton--Schulz (NS) iterations toward a near-isometric factor, encouraging global feature mixing and reducing spectral anisotropy~\cite{largeScalableOptimizationModular2024, bernsteinOldOptimizerNew2024a, bernsteinModularDualityDeep2024a}.
However, Muon-style orthogonalization predominantly constrains \emph{geometry} (directional structure), while \emph{effective step magnitudes} remain sensitive to calibration and training phase.
Moreover, real training signals are often heavy-tailed and feature-localized: transient bursts can concentrate energy on a small subset of feature axes, leading to loss spikes and a narrow stability window if magnitudes are not carefully controlled~\cite{behrouzATLASLearningOptimally2025, teamKimiK2Open2025a, parkOutlierSafePreTrainingRobust2025}.

\noindent\textbf{Why adaptive scaling?}
In practice, a single global learning rate can correspond to markedly different parameter-space update norms across layers, tensor shapes, and early-vs-late training regimes.
While near-isometric mixing can improve optimization efficiency, the lack of explicit magnitude calibration increases reliance on warmup length and schedule tuning.
This motivates a lightweight, global RMS calibration to make step sizes more comparable in parameter space and to widen the usable learning-rate region.

\noindent\textbf{Why a trust region on top of scaling?}
Global RMS calibration controls overall step magnitude but does not directly constrain \emph{energy concentration} across feature axes.
Under heavy-tailed, feature-localized bursts, a small number of columns can transiently dominate the update; in this regime, na\"ively magnitude-corrected steps may still translate burst concentration into loss spikes.
We therefore impose a \emph{relative-energy trust region} that enforces a stable bound on per-feature energy ratios, e.g., $E_j/E_{\mathrm{ref}}\le \tau$, where $E_j$ denotes column energy and $E_{\mathrm{ref}}$ is a robust reference.
This is implemented by a multiplicative damping $c_t\in[c_{\min},1]^{d_{\mathrm{in}}}$ that selectively suppresses bursty feature axes, while largely preserving the Muon-style structured mixing factor.

\noindent\textbf{Our approach.}
We propose \textbf{TrasMuon} (\textbf{T}rust-\textbf{R}egion \textbf{A}daptive \textbf{S}caling for \textbf{Muon}), which factorizes matrix updates into a structured mixing factor and lightweight magnitude controls through multiplicative coupling.
For a matrix parameter $W\in\mathbb{R}^{d_{\mathrm{out}}\times d_{\mathrm{in}}}$, TrasMuon applies
\begin{equation}
\Delta W_t \;=\; -\,\hat{\eta}_t \, O_t^{\mathrm{base}} \,\mathrm{diag}(c_t),
\qquad c_t\in[c_{\min},1]^{d_{\mathrm{in}}}.
\label{eq:intro_update}
\end{equation}
Here $O_t^{\mathrm{base}}$ is a Muon-style near-isometric update factor obtained by NS orthogonalization, combined with lightweight row-wise second-moment scaling (NorMuon-style)~\cite{liNorMuonMakingMuon2025a}.
Magnitude is controlled in two complementary ways:
(i) an RMS-calibrated step size $\hat{\eta}_t$ that reduces step-size sensitivity and improves cross-layer calibration~\cite{bernsteinModularDualityDeep2024a,largeScalableOptimizationModular2024};
and (ii) a feature-wise trust-region damping $c_t$ computed from relative column-energy ratios, which selectively suppresses bursty feature axes while largely preserving the structured mixing factor.
To reduce dependence on schedule design and warmup length, we stabilize the damping signal over time using effective-time weighting, with weights tied to the effective step size (schedule-free style averaging)~\cite{defazio2024road}.

\noindent\textbf{Contributions:}
\begin{itemize}[topsep=0pt, itemsep=1pt, parsep=0pt, partopsep=0pt]
  \item \textbf{Algorithm.} We introduce \textsc{TrasMuon}, combining Muon-style near-isometric mixing with global RMS calibration and a relative-energy trust region for feature-localized bursts, plus temporally stabilized damping via effective-time weighting.
  \item \textbf{Foundation-model training evidence.} \textsc{TrasMuon} improves early-stage convergence and stability in both warmup-enabled and warmup-free settings, indicating reduced step-size and schedule sensitivity.
  \item \textbf{Robustness.} \textsc{TrasMuon} consistently reduces loss spikes and yields stronger or more reliable final performance than relevant baselines under heavy-tailed and structured non-stationarity.
\end{itemize}

\noindent\textbf{Implication for large-scale training.}
By making matrix-structured (Muon-family) updates compatible with predictable step magnitudes and outlier-resilient dynamics, \textsc{TrasMuon} not only converges faster than vanilla Muon, but also more stable compared to other optimizers with adaptive optimizers, with reduced reliance on delicate warmup/schedule tuning.
These make TrasMuon optimizers closer to a practical, drop-in option for large-model pretraining at scale under heavy-tailed training noise.


\section{Related Work}
\label{sec:related_work}

\paragraph{Diagonal preconditioning and Adam-style optimizers.}
Adaptive methods based on inexpensive diagonal second-moment estimates, most notably Adam and AdamW~\cite{kingma2017adammethodstochasticoptimization,loshchilov2019decoupledweightdecayregularization}, are widely used due to their robustness across architectures and training regimes.
A large body of work refines Adam-style magnitude control via improved moment estimators, decay rules, or mixing strategies, aiming to better match curvature surrogates and gradient-noise statistics while retaining efficiency~\cite{yuan2024mars,pagliardini2024ademamixoptimizerbetterfaster,marfinetzEvolvingDeepLearning2025,shao2025bds,gupta2024nesterov}.
These methods primarily operate coordinate-wise, which can stabilize training but may underutilize matrix-level structure available in weight tensors (e.g., feature mixing across columns).
TrasMuon is complementary: it retains structured mixing in matrix updates while introducing axis-selective magnitude control, rather than reverting to fully diagonal update geometry.

\paragraph{Matrix and block-structured preconditioning beyond diagonal adaptivity.}
Beyond diagonal adaptivity, classical work explores richer (block-structured) preconditioning to better capture curvature while remaining computationally feasible, including Kronecker-factored curvature approximations such as K-FAC~\cite{martens2015kfac} and tensor/matrix second-moment preconditioners such as Shampoo~\cite{gupta2018shampoo}.
Related low-memory approaches (e.g., Adafactor) exploit factored second-moment structure for large models~\cite{shazeer2018adafactor}.
Another practical line controls step magnitudes via layerwise norm or trust-ratio scaling, comparing parameter and update norms as in LARS/LAMB~\cite{you2017lars,you2019lamb}.
TrasMuon differs in that it does not estimate curvature factors; instead it constructs a near-isometric mixing factor via Newton--Schulz orthogonalization and stabilizes magnitudes using global RMS calibration together with a relative-energy constraint on feature axes.

\paragraph{Orthogonalization-based directions and Muon-style updates.}
Recent work revisits non-diagonal update geometry for matrix parameters using finite-step Newton--Schulz iterations to approximate polar factors of momentum updates, producing near-orthogonal (near-isometric) directions that can improve training on modern Transformer architectures~\cite{jordan2024muon,bernstein2025deriving}.
Broader perspectives connect orthogonalized updates to modular optimization and manifold-inspired parameterizations, including variants that incorporate normalization, duality, or layerwise geometry constraints~\cite{bernsteinModularDualityDeep2024a,largeScalableOptimizationModular2024}.
A growing set of Muon-inspired variants further explores practical refinements to orthogonalized updates and their training behavior (e.g.,~\cite{pethick2025training,ahn2025dion,kumar2025curvadion,khaled2025muonbp,riabinin2025gluon,liNorMuonMakingMuon2025a}).
TrasMuon builds on this line by explicitly factorizing direction construction and magnitude stabilization: it preserves a Muon-style structured mixing factor while targeting a distinct failure mode---bursty, axis-localized energy spikes---via relative-energy damping and temporal smoothing.

\paragraph{Trust-region magnitude control, clipping, and effective-time averaging.}
Magnitude control is a long-standing stability tool, ranging from global gradient clipping and its variants~\cite{pascanu2013difficulty,brock2021agc,allouah2025adaptiveclippingFL} to trust-region methods that constrain per-step change~\cite{conn2000trust}.
In large-scale training, related magnitude-control heuristics also include trust-ratio normalization (e.g., LARS/LAMB)~\cite{you2017lars,you2019lamb} and classical adaptive scaling approaches such as AdaGrad~\cite{duchi2011adagrad}.
Such mechanisms are often global and do not identify which feature axes dominate instability.
In contrast, TrasMuon applies \emph{feature-wise} damping based on a \emph{relative-energy} constraint (e.g., $E_j/E_{\mathrm{ref}}\le\tau$), selectively suppressing bursty columns while largely preserving structured mixing; global RMS calibration further improves cross-layer step-size comparability~\cite{bernsteinModularDualityDeep2024a,largeScalableOptimizationModular2024}.
Finally, schedule-free and effective-time weighted averaging smooth training dynamics and reduce sensitivity to warmup length and training horizon~\cite{defazio2024road}, closely related to classical Polyak--Ruppert averaging~\cite{polyak1992averaging,ruppert1988rm} and multi-timescale averaging mechanisms such as SWA and Lookahead~\cite{izmailov2018swa,zhang2019lookahead}.
TrasMuon adopts effective-time weighted smoothing to stabilize feature-wise energy statistics under nonstationary bursts.

\section{Methodology}

\label{sec:methodology}
\subsection{TrasMuon Algorithm}
We propose \textbf{TrasMuon} (\textbf{T}rust-\textbf{R}egion \textbf{A}daptive \textbf{S}caling for \textbf{Muon}),  which \emph{factorizes} matrix updates into a structured mixing factor and lightweight magnitude controls (global RMS calibration and feature-wise damping), as summarized in Alg.~\ref{alg:trasmuon_highlight}.

\begin{algorithm}[ht]
    \scriptsize
  \caption{\textsc{TrasMuon}: \protect Muon + \algLHM{Adaptive Scaling} + \protect\algCP{Trust Region} + \protect\algYZ{Schedule-Free Smoothed}}
  \label{alg:trasmuon_highlight}
  \begin{algorithmic}
    \STATE {\bfseries Input:} $W\!\in\!\mathbb{R}^{d_{\mathrm{out}}\times d_{\mathrm{in}}}$, base lr $\eta$, $\beta_1,\beta_2$, $\epsilon$, NS steps $T$, weight decay $\lambda$, $c_{\min},\alpha,\beta_E,\beta_c$, trigger $k$, update period $K$, warmup $T_w$, mix $\rho$
    \STATE Initialize $M\!\leftarrow\!0$, $v^{\mathrm{row}}\!\leftarrow\!0$, $E^{\mathrm{ref}}\!\leftarrow\!0$, $c^{\mathrm{ema}}\!\leftarrow\!\mathbf{1}$, $c^{\mathrm{last}}\!\leftarrow\!\mathbf{1}$, $S\!\leftarrow\!0$, $C\!\leftarrow\!\mathbf{0}$, $\gamma_t \!\leftarrow\! \eta$
    \REPEAT
      \STATE $G \leftarrow \nabla_W \mathcal{L}(W)$ ;\;\; $W \leftarrow (1-\eta\lambda)W$
      \STATE {\bfseries Momentum:} $M \leftarrow \beta_1 M + (1-\beta_1)G$
      \STATE {\bfseries Orthogonalized direction:} $O \leftarrow \mathrm{NS}(\tilde{M};T)$
      \STATE $v^{\mathrm{row}} \leftarrow \beta_2 v^{\mathrm{row}} + (1-\beta_2)\,\mathrm{mean}_{j}(O_{\cdot j}^{2})$
      \STATE $O^{\mathrm{base}} \leftarrow \mathrm{diag}\!\big((v^{\mathrm{row}}+\epsilon)^{-1/2}\big)\,O$

      \STATE \algLHM{{\bfseries Calibration:} $\hat{\eta} \leftarrow \eta\,\dfrac{\sqrt{d_{\mathrm{out}}d_{\mathrm{in}}}}{\|O^{\mathrm{base}}\|_F+\epsilon}$}

      \STATE \algCP{{\bfseries Column energy:} $E_j \leftarrow \sum_i M_{ij}^2$}
      \STATE \algCP{{\bfseries Robust reference:} $E^{\mathrm{cur}} \leftarrow \mathrm{Quantile}_{0.5}(\{E_j\})$}
      \STATE \algCP{{\bfseries EMA smooth:} $E^{\mathrm{ref}} \leftarrow \beta_E E^{\mathrm{ref}} + (1-\beta_E)E^{\mathrm{cur}}$}
      \STATE \algYZw{{\bfseries Schedule-free accumulators:} $S \leftarrow S+\gamma_t^2$, $C \leftarrow C+\gamma_t^2 c^{\mathrm{last}}$, $c^{\mathrm{avg}} \leftarrow C/(S+\epsilon)$}

      \IF{$t>T_w$ {\bfseries and} $t\bmod K=0$}
        \STATE \algCP{$r_j \leftarrow E_j/(E^{\mathrm{ref}}+\epsilon)$}
        \STATE \algCP{$c_j^{\mathrm{clip}} \leftarrow \mathrm{clip}\Big( \frac{1}{1+\alpha\log(1+r_j)},\,c_{\min},\,1\Big)$}
        \STATE \algCP{{\bfseries Trigger (optional):} $c_j^{\mathrm{inst}} \leftarrow c_j^{\mathrm{clip}}$ if $r_j>k$ else $1$}
        \STATE \algCP{{\bfseries EMA smooth:} $c^{\mathrm{ema}} \leftarrow \beta_c c^{\mathrm{ema}} + (1-\beta_c)c^{\mathrm{inst}}$}
      \ENDIF
      \STATE \algYZw{{\bfseries Long-Short term Mixing:} $c \leftarrow \mathbf{1}$ if $t\le T_w$ else $(1-\rho)c^{\mathrm{ema}}+\rho c^{\mathrm{avg}}$}
      \STATE \algYZw{$c^{\mathrm{last}} \leftarrow c$ \COMMENT{cached between updates}}

      \STATE {\bfseries Update:} $W \leftarrow W - \hat{\eta}\,\big(O^{\mathrm{base}}\odot \mathrm{ExpandCols}(c)\big)$
    \UNTIL{training ends}
  \end{algorithmic}
\end{algorithm}

For a matrix parameter $W\!\in\!\mathbb{R}^{d_{\mathrm{out}}\times d_{\mathrm{in}}}$ with gradient
$G_t=\nabla_W\mathcal{L}(W_t)$, TrasMuon applies the multiplicative update
\begin{equation}
\Delta W_t \;=\; -\,\hat{\eta}_t \, O_t^{\mathrm{base}} \,\mathrm{diag}(c_t),
\label{eq:trasmuon_update_main}
\end{equation}
where $O_t^{\mathrm{base}}$ is a structured Muon-style direction, $\hat{\eta}_t$ is an RMS-calibrated  step size,
and $c_t\in[c_{\min},1]^{d_{\mathrm{in}}}$ is a damping-only feature (column-wise) clip.

\paragraph{Orthogonalized direction.}
We maintain momentum
\begin{equation}
M_t=\beta_1 M_{t-1}+(1-\beta_1)G_t,
\label{eq:momentum_main}
\end{equation}
remove scale via RMS pre-normalization
\begin{equation}
\tilde{M}_t \;=\; \frac{M_t}{\|M_t\|_F/\sqrt{d_{\mathrm{out}}d_{\mathrm{in}}}+\epsilon},
\label{eq:prenorm_main}
\end{equation}
and approximate the polar factor by $T$ Newton-Schulz steps,
\begin{equation}
O_t \;\approx\; \mathrm{NS}(\tilde{M}_t;T),
\label{eq:ns_main}
\end{equation}
yielding a near-isometric direction robust to axis rotations.

\paragraph{Row-wise scaling calibration.}
We apply lightweight row-wise second-moment scaling (NorMuon-style~\cite{liNorMuonMakingMuon2025a})
\begin{equation}
v_t^{\mathrm{row}}=\beta_2 v_{t-1}^{\mathrm{row}}+(1-\beta_2)\,\mathrm{mean}_{j}(O_{t,\cdot j}^{\odot 2}),
\label{eq:rowscale_main}
\end{equation}
\begin{equation}
O_t^{\mathrm{base}}=\mathrm{diag}\!\big((v_t^{\mathrm{row}}+\epsilon)^{-1/2}\big)\,O_t.
\label{eq:rowscale_main1}
\end{equation}
We also calibrate the global step size by,
\begin{equation}
\hat{\eta}_t
\;=\;
\eta \cdot \frac{\sqrt{d_{\mathrm{out}}d_{\mathrm{in}}}}{\|O_t^{\mathrm{base}}\|_F+\epsilon},
\label{eq:tr_cal_main}
\end{equation}
which bounds the update RMS as $\|\Delta W_t\|_F/\sqrt{mn}\le \eta$ (since $c_t\le\mathbf{1}$), reducing sensitivity to layer shape and transient fluctuations~\cite{bernsteinModularDualityDeep2024a,largeScalableOptimizationModular2024}. 
Row-wise scaling serves as lightweight local conditioning, while global RMS calibration and trust-region damping provide the primary magnitude stabilization.

\paragraph{Trust-region clipping.}
This implements a relative-energy trust region by controlling the per-column ratio $r_{t,j}=E_{t,j}/(E_t^{\mathrm{ref}}+\epsilon)$, ensuring burst-dominated columns receive multiplicative damping. In particular, damping enforces an implicit constraint of the form $r_{t,j}\lesssim\tau$ in the effective update, for a tunable tolerance controlled by $\alpha$ (and optional trigger $k$).

We measure column energies on the pre-orthogonalization momentum $M_t$ to detect feature-localized bursts before NS mixing compresses global magnitudes and redistributes energy across axes.
\begin{equation}
E_{t,j}=\sum_{i=1}^{d_{\mathrm{out}}} M_{t,ij}^2,
\end{equation}
\begin{equation}
E_t^{\mathrm{cur}}=\mathrm{Quantile}_{0.5}(\{E_{t,j}\}),
\end{equation}
\begin{equation}
E_t^{\mathrm{ref}}=\beta_E E_{t-1}^{\mathrm{ref}}+(1-\beta_E)E_t^{\mathrm{cur}},
\label{eq:eref_main}
\end{equation}
and form $r_{t,j}=E_{t,j}/(E_t^{\mathrm{ref}}+\epsilon)$.
Using the median yields a high-breakdown reference, preventing sparse bursts from arbitrarily inflating $E_t^{\mathrm{ref}}$~\cite{hampel1986robust,huber1981robust}.
We then apply a smooth, damping-only clip
\begin{equation}
c_{t,j}^{\mathrm{raw}}=\frac{1}{1+\alpha\log(1+r_{t,j})},
\end{equation}
\begin{equation}
c_{t,j}^{\mathrm{clip}}=\mathrm{clip}\!\left(c_{t,j}^{\mathrm{raw}},c_{\min},1\right),
\label{eq:clip_main}
\end{equation}
(optionally triggered only when $r_{t,j}>k$), and set $c_t$ by temporal smoothing (more details seen in Appendix \ref{app:methodology}).

\paragraph{Schedule-free temporal smoothing.}
We smooth the instantaneous clip $c_t^{\mathrm{inst}}$ with EMA
\begin{equation}
c_t^{\mathrm{ema}}=\beta_c c_{t-1}^{\mathrm{ema}}+(1-\beta_c)c_t^{\mathrm{inst}},
\label{eq:ema_main}
\end{equation}
and optionally apply schedule-free averaging~\cite{defazio2024road} using $\gamma_t$ (default $\gamma_t=\eta$):
\begin{equation}
S_t=S_{t-1}+\gamma_t^2,\qquad
C_t=C_{t-1}+\gamma_t^2 c_{t-1}^{\mathrm{last}},
\end{equation}
\begin{equation}
c_t^{\mathrm{avg}}=\frac{C_t}{S_t+\epsilon},\qquad
c_t=(1-\rho)c_t^{\mathrm{ema}}+\rho c_t^{\mathrm{avg}}.
\label{eq:sf_main}
\end{equation}
We cache $c_t^{\mathrm{last}}\!\leftarrow\!c_t$ between gate updates to avoid bias when the raw clip is computed every $K$ steps. 
When $\gamma_t$ is constant, $c_t^{\mathrm{avg}}$ reduces to a long-horizon running average; when $\gamma_t$ varies (e.g., due to warmup/decay), the $\gamma_t^2$ weights align the smoothing with the effective step size as in schedule-free optimization~\cite{defazio2024road}.

\subsection{Convergence Analysis}
\label{sec:convergence}
We briefly summarize deterministic guarantees and a standard convergence framework for TrasMuon; full statements and proofs are deferred to Appendix~\ref{app:convergence}.
Recall $\Delta W_t=-\hat{\eta}_t U_t$ with $U_t=O_t^{\mathrm{base}}\mathrm{diag}(c_t)$ and $c_{t,j}\in[c_{\min},1]$.

\paragraph{Damping-only contraction.}
For any matrix $A$ and any $c\in[0,1]^n$, right-multiplication by $\mathrm{diag}(c)$ cannot increase the Frobenius norm:
\begin{equation}
\|A\,\mathrm{diag}(c)\|_F \le \|A\|_F.
\label{eq:damping_contraction_main}
\end{equation}

\paragraph{RMS calibration.}
With $\hat{\eta}_t=\eta\sqrt{mn}/(\|O_t^{\mathrm{base}}\|_F+\epsilon)$ and damping-only $c_t\le \mathbf{1}$, the update norm is uniformly bounded:
\begin{equation}
\|\Delta W_t\|_F \le \eta\sqrt{mn}\qquad \forall t,
\label{eq:tr_bound_main}
\end{equation}
independently of transient gradient spikes.

\paragraph{Stationarity under smoothness and alignment.}
Under standard $L$-smoothness and a mild alignment condition (Appendix~\ref{app:convergence}), TrasMuon satisfies an expected first-order stationarity bound of the form
\begin{equation}
\frac{1}{T}\sum_{t=0}^{T-1}\mathbb{E}\|\nabla f(W_t)\|_F^2
\;\le\;
\frac{\mathbb{E}[f(W_0)]-f^\star}{\mu\,\eta\,T}
\;+\;
\frac{L}{2\mu}\,\eta\,mn,
\label{eq:stationarity_main}
\end{equation}
for a constant $\mu>0$ capturing effective descent. Importantly, the EMA/schedule-free construction of $c_t$ only affects how the clip is computed, while the theory relies solely on the invariant $c_{t,j}\in[c_{\min},1]$.

\section{Experiments}

\subsection{Language Model Pretraining: Fast Early Descent and Late-Stage Dynamics}
\label{sec:QWEN}

\paragraph{Language model training setup.}
\textsc{TrasMuon} is evaluated in a controlled, short-run pretraining-style setting and compared against four baseline optimizers:
AdamW~\cite{loshchilov2019decoupledweightdecayregularization}, Muon~\cite{jordan2024muon}, Dion~\cite{ahn2025dion}, and NorMuon~\cite{liNorMuonMakingMuon2025a}.
All methods train decoder-only Transformer models from random initialization, including GPT-2~\cite{gpt2} and Qwen3-0.6B~\cite{yangQwen3TechnicalReport2025}, on FineWeb-Edu~\cite{fineweb-edu}.
Each run is executed for $1500$ optimization steps with sequence length $1024$ and a fixed global (effective) batch size of $1024$ across models and optimizers, corresponding to
$1500 \times 1024 \times 1024 \approx 1.57 \times 10^9$ training tokens.

\paragraph{Optimization hyperparameters and sensitivity.}
To isolate optimizer behavior under an identical training budget, all methods share the same learning rate $\eta=3.6\times 10^{-3}$ and weight decay $\lambda=5\times 10^{-3}$.
All runs follow the same default seeding behavior of the training stack.
For NorMuon in particular, the most influential optimizer-specific hyperparameter is the RMS normalization, which controls scale calibration of the orthogonalized update.
This section is treated as a fixed-hyperparameter pilot, where broader sweeps and multi-seed evaluations are left to future work.

\paragraph{Learning-rate schedule and reporting.}
A warmup--stable--decay learning-rate schedule is adopted, and results are reported for both warmup-enabled and warmup-free configurations. All other training components, including the data pipeline, batching, tokenization, model architecture, and compute budget, are kept identical across optimizers.
\begin{figure}[htb]
  \vskip 0.1in
  \centering
    \begin{subfigure}[b]{\columnwidth}
    \centering
    \includegraphics[width=0.9\linewidth]{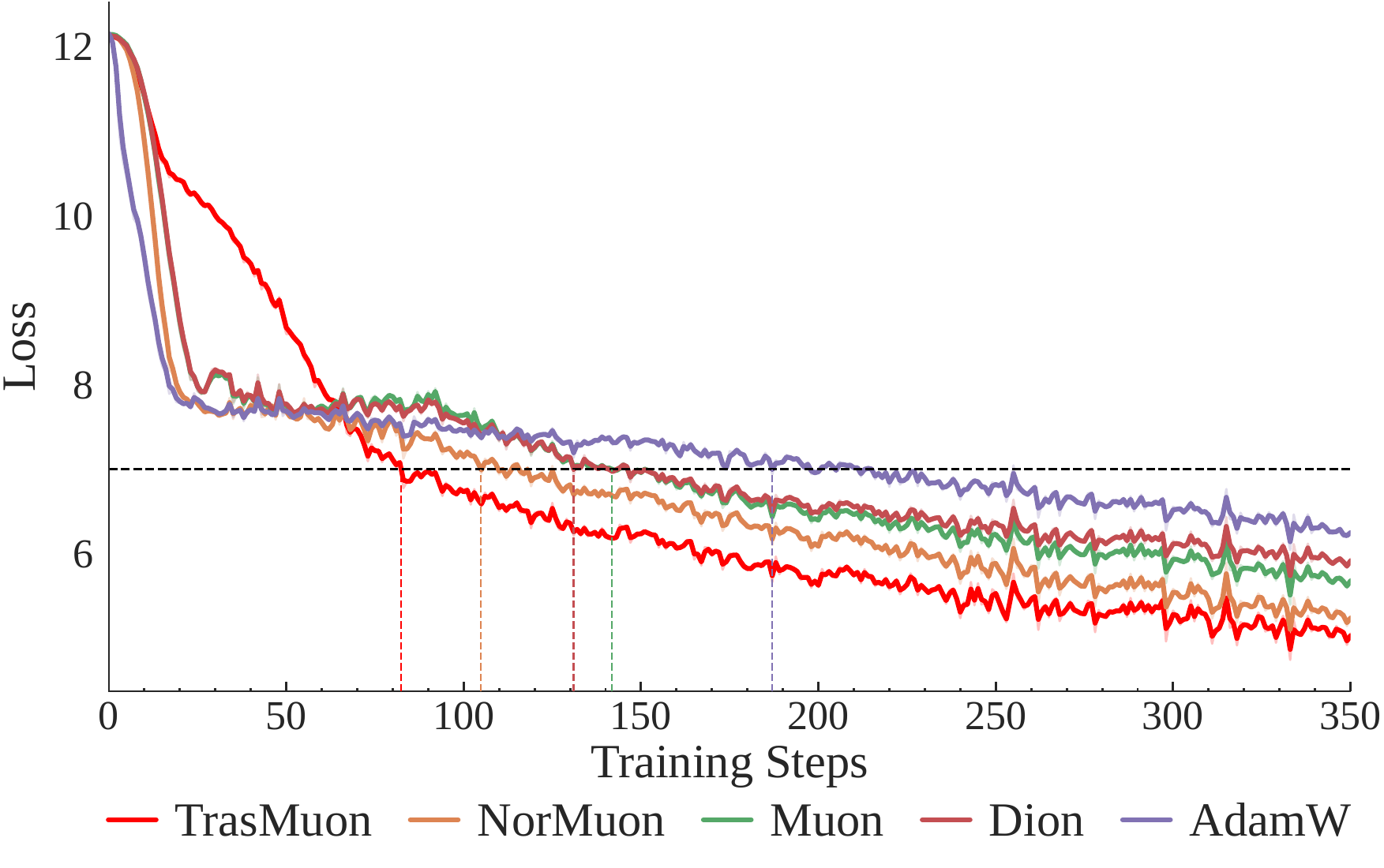}
    \caption{Qwen3-0.6B with warmup}
    \label{fig:qwen_warm01}
  \end{subfigure}
  \begin{subfigure}[b]{\columnwidth}
    \centering
    \includegraphics[width=0.9\linewidth]{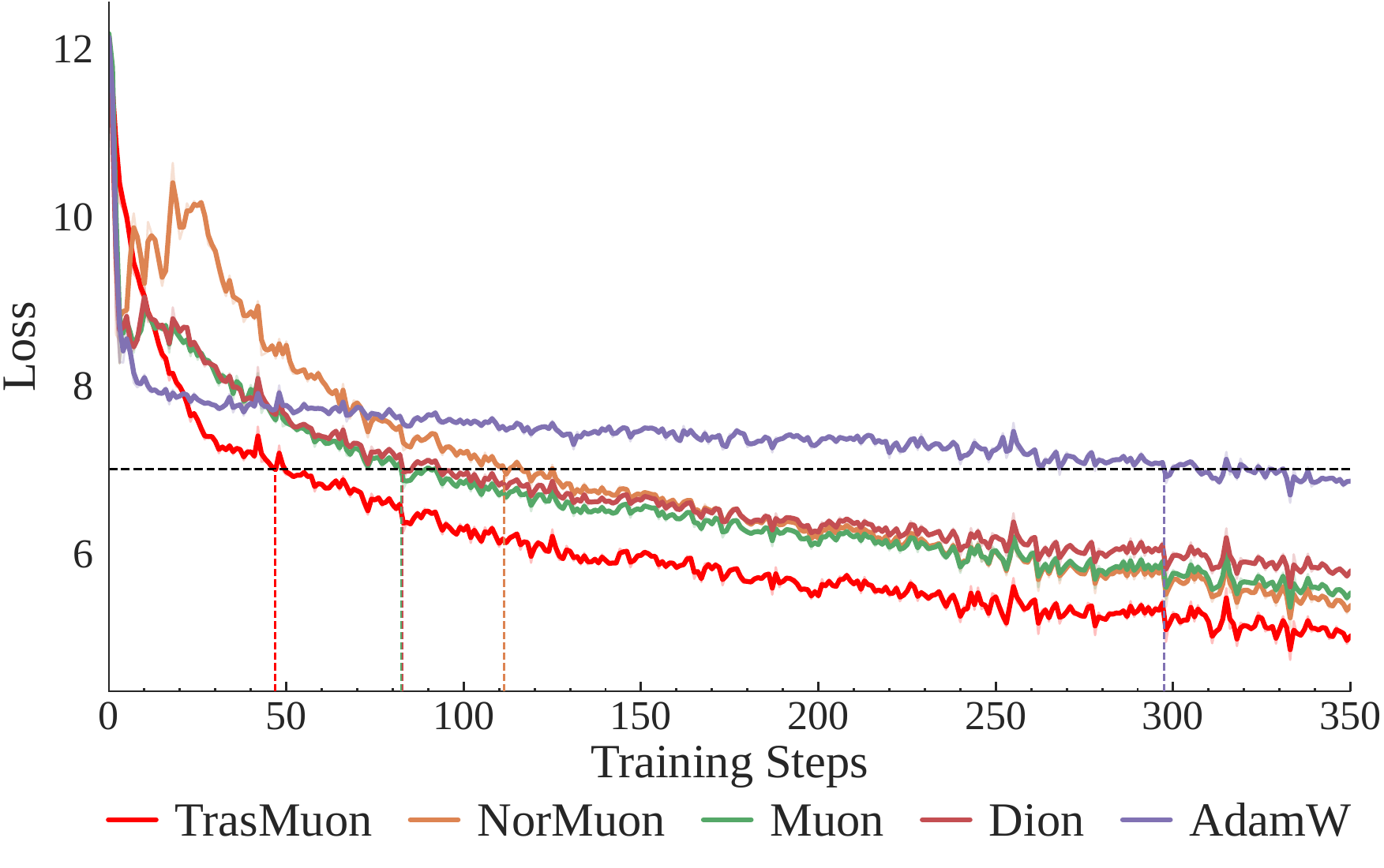}
    \caption{Qwen3-0.6B without warmup}
    \label{fig:qwen_nowarm}
  \end{subfigure}
  \vspace{0.01in} 
\caption{
Early-stage training dynamics for Qwen3-0.6B from scratch (steps 0--350) under a warmup-stable-decay schedule, comparing (a) warmup-enabled runs and (b) warmup-free runs. The curves are smoothed using a time-weighted exponential moving average (EMA) with smoothing factor $0.1$ for better visualization.
}
  \label{fig:qwen_warmup_compare}
\end{figure}

\paragraph{Training with warmup.}
Under the warmup-enabled schedule, all evaluated optimizers exhibit stable training in this fixed-hyperparameter pilot.
Although \textsc{TrasMuon} is not the fastest within the first few dozen steps, it begins to reduce loss more rapidly after approximately $80$ steps (illustrated in Fig.~\ref{fig:qwen_warm01}).
At a reference threshold of training loss ($=7.0$), \textsc{TrasMuon} reaches the target in $80$ steps, compared to $188$ steps for AdamW ($2.35\times$) and $140$ steps for Muon ($1.75\times$).\footnote{We use loss $=7.0$ as a representative checkpoint.}

\paragraph{Training without warmup.}
Without warmup, optimizer becomes more sensitive to step-size calibration.
Under the same shared learning rate and batch/sequence configuration, \textsc{TrasMuon} maintains a smooth loss trajectory in the early stage, while several baselines show larger loss oscillations as shown in Fig.~\ref{fig:qwen_nowarm}.
At the same reference threshold of training loss, \textsc{TrasMuon} reaches the target in $48$ steps, versus $298$ for AdamW ($6.21\times$) and $83$ for Muon ($1.73\times$).
Despite the widespread use of warmup in pretraining, this setting remains meaningful, as the choice of warmup length is typically determined heuristically.

\paragraph{Late-stage behavior under a fixed budget.}
Beyond early-stage speed, \textsc{TrasMuon} also attains the lowest (or comparable-lowest) training loss in the late-stage window under both warmup settings in this pilot run.
Late-stage curves and loss comparisons are reported in Fig.~\ref{fig:qwen_last_steps} of Appendix~\ref{app:Qwen3_last_step}.
The early-stage gap narrows as training progresses into a slower-loss regime, suggesting that the benefit is strongest when training dynamics are most nonstationary.

\noindent\textbf{Decreasing feature-wise energy concentration over training.}
A plausible explanation is that early training exhibits stronger feature-wise anisotropy, where a small subset of hidden dimensions (columns) contributes disproportionately to gradient or momentum energy.
In this regime, \textsc{TrasMuon}'s energy-based feature-wise clipping is engaged, selectively damping bursty columns and improving stability without discarding the structured Muon direction.
As representations become better calibrated, energy may become more uniformly distributed across feature axes, reducing the prevalence of strongly concentrated column-localized bursts; consequently, the clipping signal weakens and the effective update becomes closer to the NorMuon backbone.
This explanation is offered as a hypothesis and a direct characterization of activation/gradient anisotropy (e.g., via tracked energy ratios and gate statistics over time) is left to future work.


\subsection{Vision Transformer Experiments}
\label{sec:vit_experiments}

\paragraph{ImageNet-100: Vision Transformer Training.}
We evaluate the benefits of TrasMuon in a large-scale vision setting by training ViT-Base~\citep{dosovitskiy2020vit} on ImageNet-100 due to limited computational resources, a 100-class subset of ImageNet-1k (ILSVRC-2012)~\citep{deng2009imagenet}, followed the standard ILSVRC-2012 train/validation protocol.
Dataset construction, class specification, and implementation details for ViT training are provided in Appendix~\ref{app:imagenet100_data_source} and~\ref{app:imagenet100_exp_protocol}.
We compared AdamW, Muon, NorMuon, and \textsc{TrasMuon}, using identical training budgets and hyperparameters. Results show training loss and \emph{validation} top-1 accuracy, aggregated over three 
random seeds.

\begin{figure}[th]
    \centering
    \begin{subfigure}{0.49\textwidth}
        \centering
        \includegraphics[width=\linewidth]{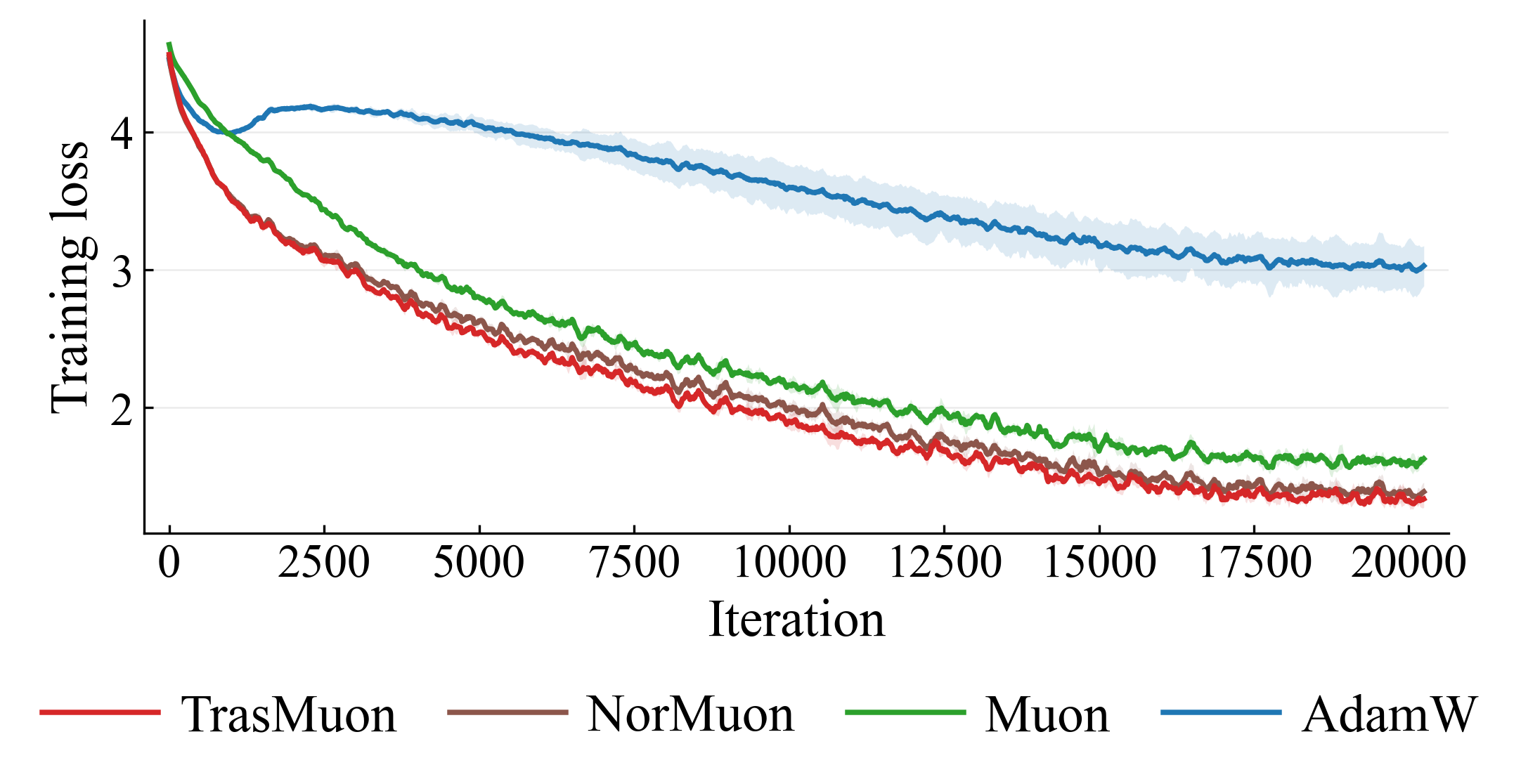}
        \caption{Training loss vs.\ iteration.}
        \label{fig:vit-imagenet100-loss}
    \end{subfigure}\hfill
    \begin{subfigure}{0.49\textwidth}
        \centering
        \includegraphics[width=\linewidth]{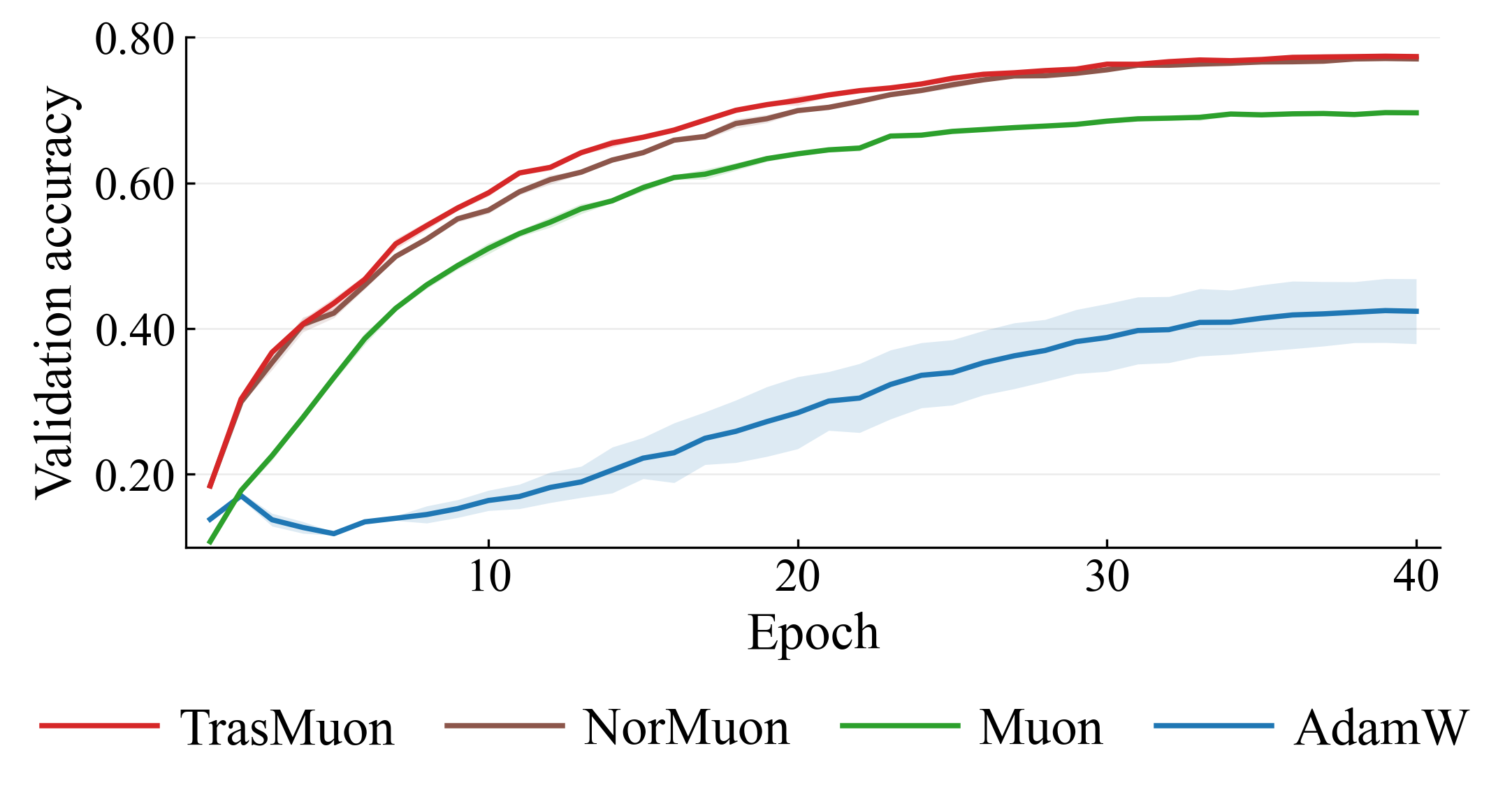}
        \caption{Validation top-1 accuracy vs.\ epoch.}
        \label{fig:vit-imagenet100-acc}
    \end{subfigure}
    \caption{ViT-Base training on ImageNet-100.
    Multi-seed results (mean $\pm$ std over three seeds: 42, 43, 44) for 4 optimizers.
    Shaded regions denote variability across seeds.}
\end{figure}

\noindent\textbf{Results.}
Across all optimizers evaluated with multi-seed runs, Muon, Normuon, and  \textsc{TrasMuon} consistently improve optimization behavior and validation accuracy over AdamW in these experiments, demonstrating the advantage of optimizers, which update based on structured, near-orthogonal update directions.
 illustrated in Fig.~\ref{fig:vit-imagenet100-loss} and Fig.~\ref{fig:vit-imagenet100-acc}, \textsc{TrasMuon} achieves the fastest loss reduction, the highest validation accuracy, and reduced variability across different seeds and various compared optimizers. 
 Moreover, CIFAR-100 has been evaluated robustness under controlled column-localized burst injection in Appendix~\ref{app:vit_cifar_burst}.


\subsection{PINNs Benchmark: ROI Sampling as a Nonstationary Stress Test}
\label{sec:pinns}

Adaptive collocation in physics-informed neural networks (PINNs) is often necessary to resolve localized errors and stiff PDE behavior, where uniform sampling under-resolves difficult regions~\cite{gao2023failure,subramanian2022adaptive,wu2023comprehensive}.
Here we use region-of-interest (ROI) densification as a \emph{controlled nonstationarity} mechanism to stress-test optimizer robustness: periodically concentrating interior collocation points in a small subregion induces distribution shifts in the residual samples, perturbing gradient statistics in a reproducible way.

\paragraph{Helmholtz equation setup.}
On $\Omega=[0,1]^2$, we consider
\begin{equation}
\Delta u(x) + \kappa^2 u(x) = f(x), \qquad u|_{\partial\Omega}=0,
\label{eq:helmholtz_pde_main}
\end{equation}
with the manufactured solution
\begin{equation}
u^\star(x,y)=\sin(\pi k x)\sin(\pi k y), \qquad \kappa=\pi k.
\label{eq:helmholtz_u_star_main}
\end{equation}
which yields $f(x)=-(\pi k)^2 u^\star(x)$.

We train an MLP $u_\theta:\Omega\to\mathbb{R}$ by minimizing
\begin{equation}
\mathcal{L}(\theta)=\mathbb{E}\left[\tfrac12 r_\theta(x)^2\right]
\;+\;\lambda_b\,\mathbb{E}\left[\tfrac12\big(u_\theta(x)-u^\star(x)\big)^2\right]
\label{eq:pinn_loss1}
\end{equation}
\begin{equation}
r_\theta(x)=\Delta u_\theta(x)+\kappa^2u_\theta(x)-f(x).
\label{eq:pinn_loss2}
\end{equation}
and report the relative error on a fixed evaluation grid
\begin{equation}
\mathrm{rel}\text{-}L_2(u_\theta,u^\star)=\frac{\|u_\theta-u^\star\|_{2}}{\|u^\star\|_{2}}.
\label{eq:rel_l2_main}
\end{equation}

\paragraph{ROI sampling protocol.}
To emulate adaptive densification, we use a time-varying interior sampling distribution
\begin{equation}
p_t(x) = (1-\alpha_t)\,p_0(x) + \alpha_t\,p_{\mathrm{roi}}(x),
\label{eq:mix_sampling_main}
\end{equation}
where $p_0$ is uniform over $\Omega$ and $p_{\mathrm{roi}}$ is uniform over $\Omega_{\mathrm{roi}}=[x_0,x_1]\times[y_0,y_1]$. 
We run $4000$ optimization steps and begin nonstationary ROI events after step $1000$.
At each ROI event (every $20$ steps), $95\%$ of interior points are sampled from a small ROI patch and the rest from the full domain. 
To avoid sensitivity to a single ROI location, the patch is drawn from a fixed pool (corners, edges, and interior patches) using a step-dependent hash for reproducibility.
This produces repeated, time-varying distribution shifts that mimic practical ROI refinement policies.

\begin{figure}[htb]
    \centering
    \begin{subfigure}{0.49\textwidth}
        \centering
        \includegraphics[width=\linewidth]{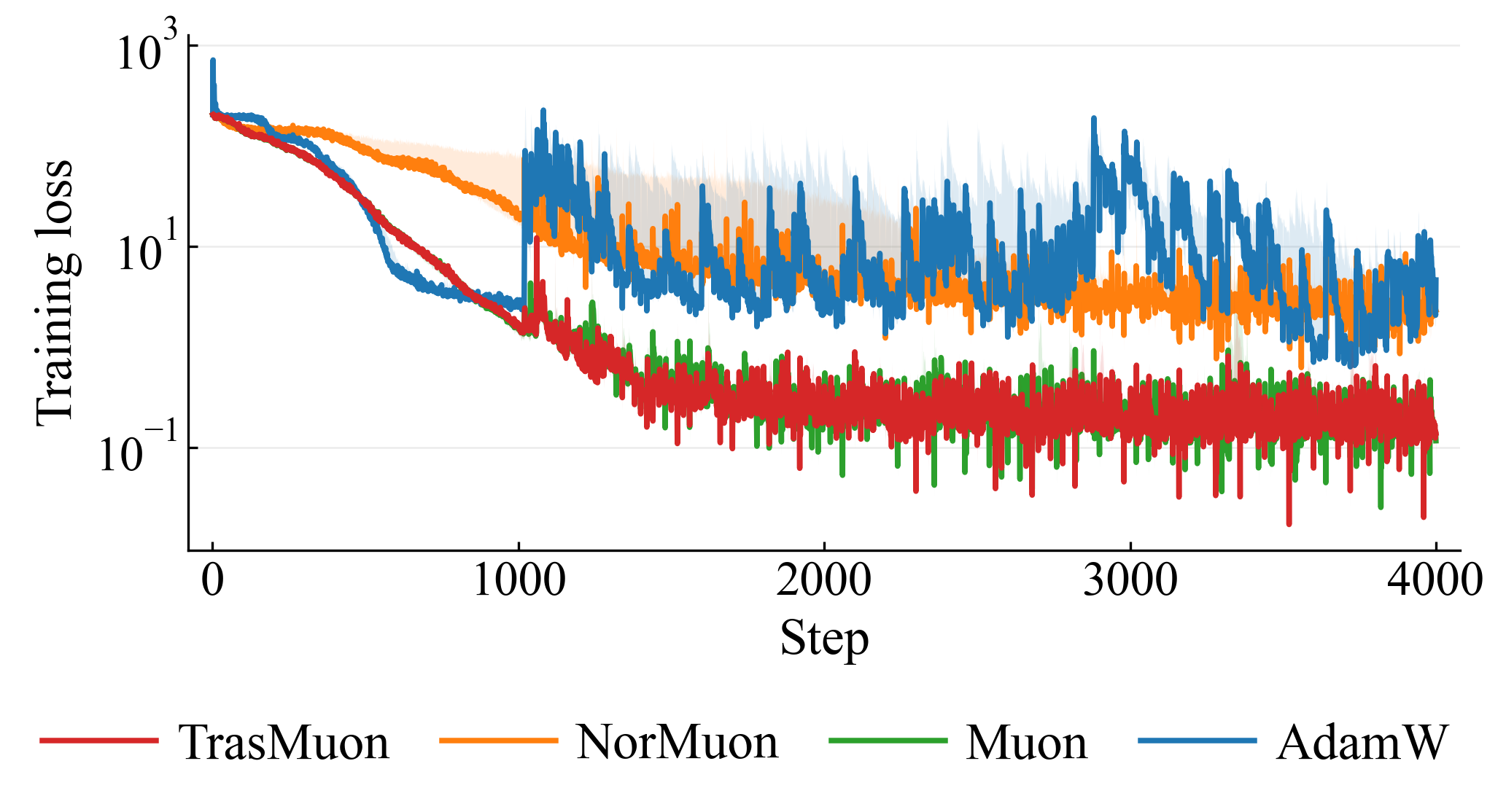}
        \caption{training objective (estimated on the time-varying sampling distribution $p_t$)}
        \label{fig:pinns_loss}
    \end{subfigure}\hfill
    \begin{subfigure}{0.49\textwidth}
        \centering
        \includegraphics[width=\linewidth]{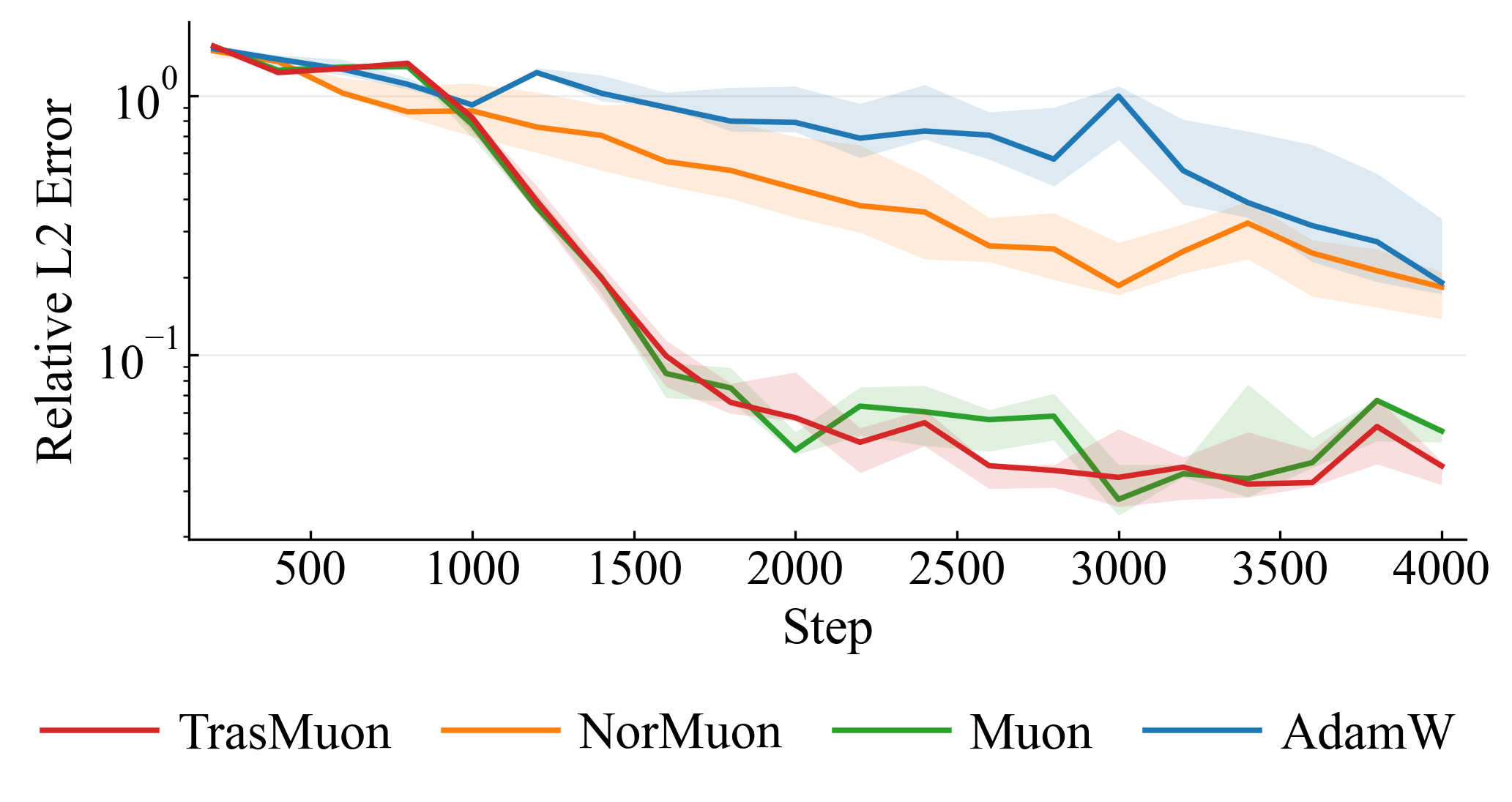}
        \caption{domain-wide relative $L_2$ error evaluated on a fixed grid. ROI events start at step $1000$ and repeat every $20$ steps}
        \label{fig:pinns_acc}
    \end{subfigure}
\caption{\textbf{PINN Helmholtz ($k{=}2$) under random ROI sampling shifts.} Curves show the mean over seeds, and shaded regions indicate variability across seeds.}
\end{figure}

\paragraph{PINN ROI-sampling stress test: convergence and robustness.}
Figure~\ref{fig:pinns_loss} and \ref{fig:pinns_acc} compare Muon and \textsc{TrasMuon} on Helmholtz ($k{=}2$) under a controlled nonstationary ROI-sampling protocol, where ROI events start at step $1000$ and recur every $20$ steps.
During the initial stationary phase (before ROI events), both methods exhibit nearly identical optimization trajectories in terms of training loss and domain-wide relative $L_2$ error, indicating that \textsc{TrasMuon} does not incur a measurable overhead or degradation under standard uniform sampling.

After ROI events begin, the training objective becomes significantly more variable due to the induced distribution shifts in interior collocation points.
In this nonstationary regime, \textsc{TrasMuon} maintains comparable or slightly lower training loss while exhibiting reduced extreme fluctuations, consistent with its design goal of suppressing bursty, feature-localized updates. These results support the conclusion that \textsc{TrasMuon} preserves baseline convergence under stationary sampling, while improving stability and final solution accuracy under controlled, nonstationary ROI sampling shifts.


\subsection{Mechanistic Study: Column-Localized Outliers and Energy-Based Feature Clipping}
\label{sec:toy2}

We design a controlled toy problem to validate the \emph{feature-wise clipping} mechanism in \textsc{TrasMuon} under intermittent, column-localized bursts.
We optimize a matrix parameter $W\in\mathbb{R}^{d\times d}$ under the quadratic objective
\begin{equation}
\min_W\; f(W) \;=\; \tfrac12 \|A W B - T\|_F^2,
\label{eq:toy2_obj}
\end{equation}
where $A=U\Sigma_AU^\top$ and $B=V\Sigma_BV^\top$ with random orthogonal $U,V$, and diagonal spectra $\Sigma_A,\Sigma_B$ chosen to yield a target condition number $\kappa\in\{10^2,10^4,10^6\}$.
This construction controls stiffness while allowing us to randomize nuisance rotations across runs.

\paragraph{Column-localized outlier injection (stress protocol).}
To emulate rare, feature-localized gradient domination, every $K_{\mathrm{out}}$ steps we inject an outlier event that amplifies a small subset of columns in a fixed feature basis.
Concretely, for momentum $M_t$ we select a set $\mathcal{J}$ of $s\ll d$ column indices and apply a multiplicative burst
\begin{equation}
\widetilde{M}_{t,\cdot j} \;=\;
\begin{cases}
a\,M_{t,\cdot j}, & \text{if } j\in\mathcal{J},\\[2pt]
M_{t,\cdot j}, & \text{otherwise},
\end{cases}
\label{eq:toy2_burst}
\end{equation}

with burst amplitude $a>1$.
This perturbation produces abrupt increases in column energy $E_{t,j}=\sum_i \widetilde{M}_{t,ij}^2$ while leaving the underlying objective \eqref{eq:toy2_obj} unchanged.

\paragraph{Preserving feature semantics.}
Because \textsc{TrasMuon}'s clipping is axis-aligned (column-wise), the stress protocol is evaluated under a \texttt{fix\_V=True} setting, i.e., the column basis is preserved across training and across injected events.
We additionally report a boundary condition where the column basis is randomized (\texttt{fix\_V=False}); in that case, injected energy disperses across columns and feature-wise clipping is not expected to yield an advantage (Appendix~\ref{app:toy2_supp}).

\paragraph{Metrics.}
We track (i) spike count and (ii) final objective value, reporting median and IQR over multiple seeds/rotations.

\paragraph{Closed-loop response.}
Figure~\ref{fig:toy2_loss_zoom} shows that \textsc{TrasMuon} reduces burst-induced loss spikes and improves convergence relative to the NorMuon backbone under matched compute.
Figure~\ref{fig:toy2_closed_loop} provides mechanism-level evidence consistent with a closed-loop response:
outlier events increase the relative column-energy ratio (e.g., $r_{q95}$/$r_{\max}$), which is immediately followed by a decrease in the \emph{applied} clipping signal (tracked by $c_{\mathrm{used,min}}$), thereby damping the burst and suppressing spikes.

\begin{figure}[t]
  \centering
  \includegraphics[width=\columnwidth]{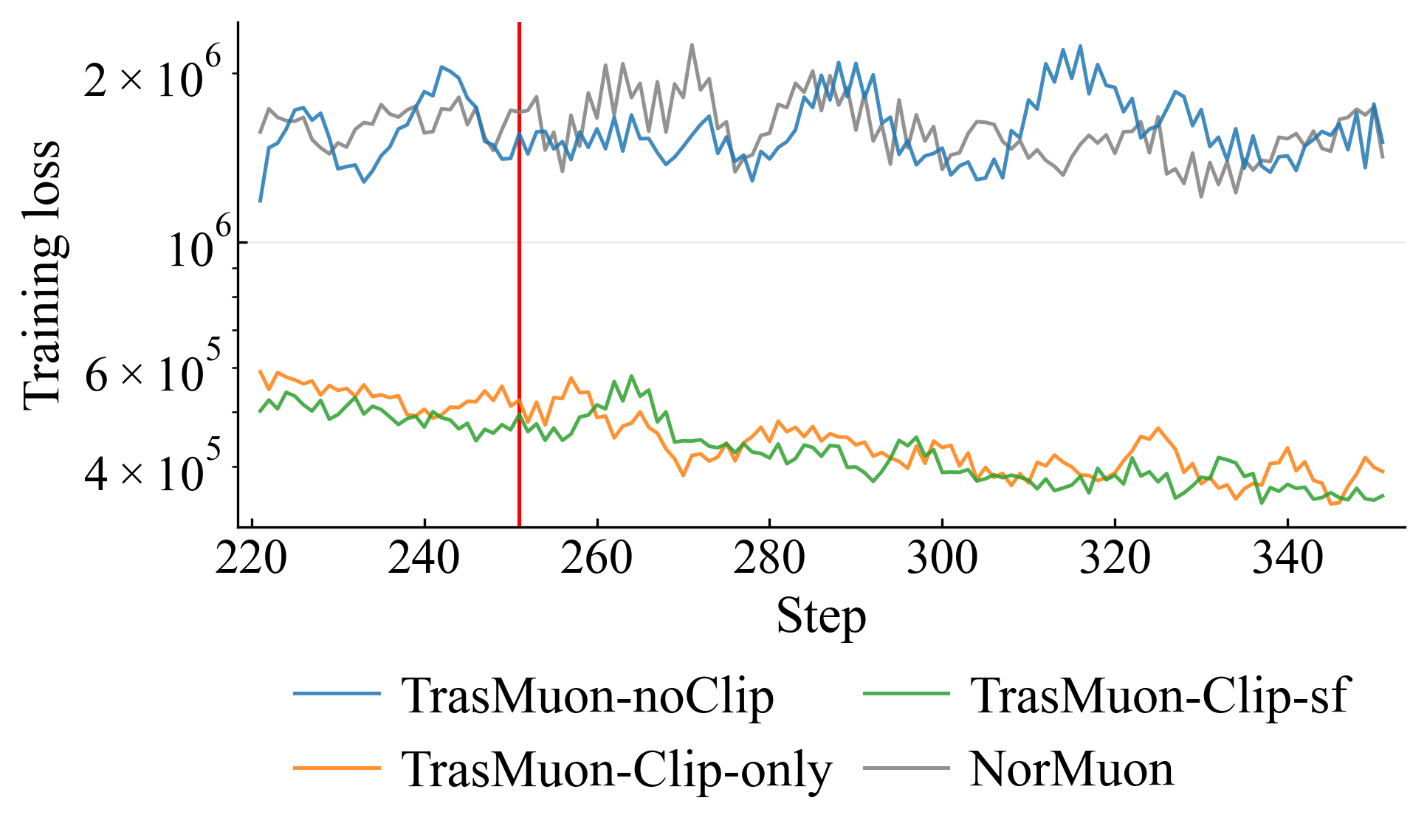}
  \caption{\textbf{Column outlier injection.}
  Loss trajectories in a window around an outlier event. Vertical markers indicate outlier steps.}
  \label{fig:toy2_loss_zoom}
\end{figure}

\begin{figure}[htb]
  \centering
  \includegraphics[width=0.9\columnwidth]{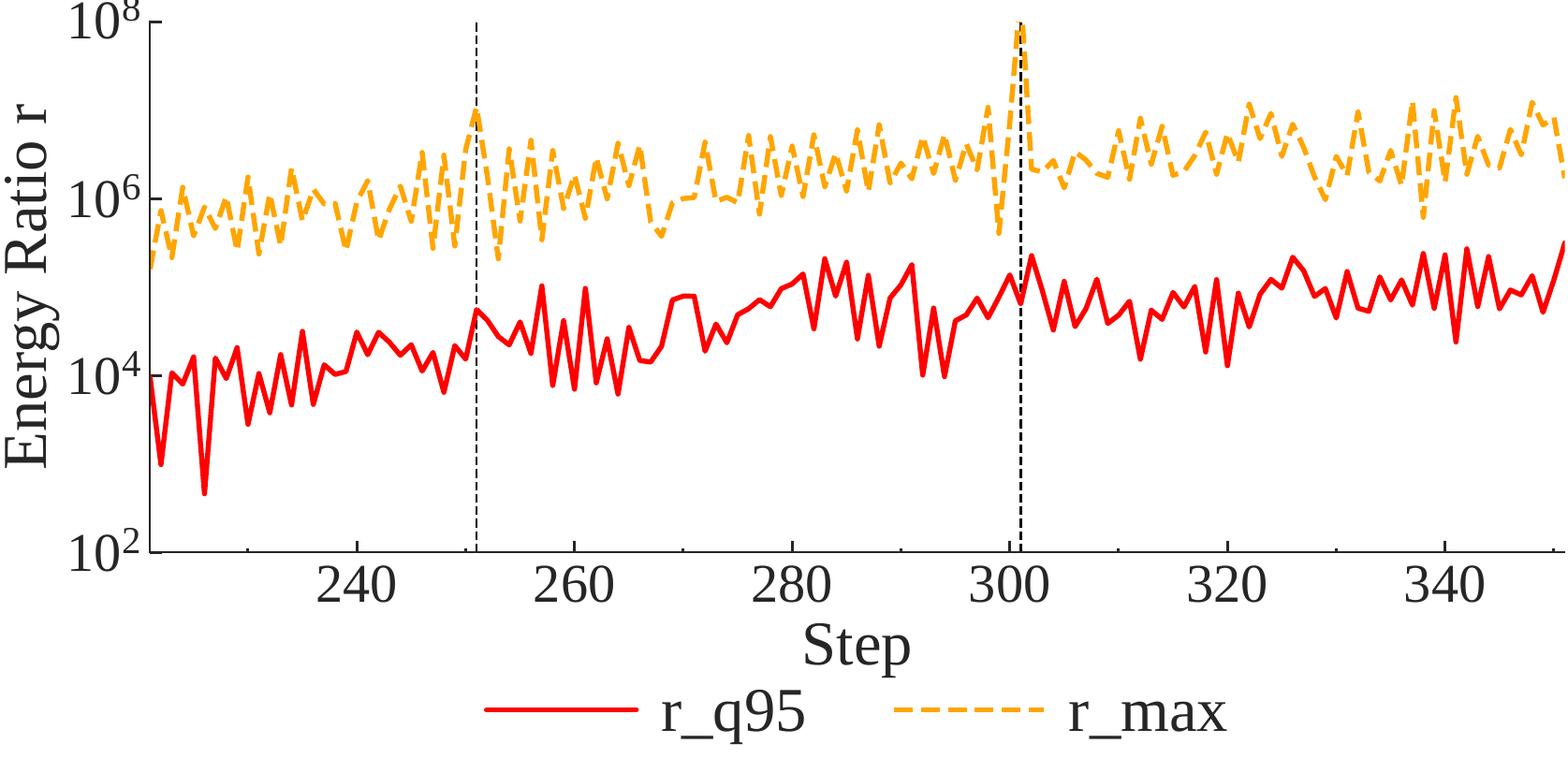}\\[0.10in]
  \includegraphics[width=0.9\columnwidth]{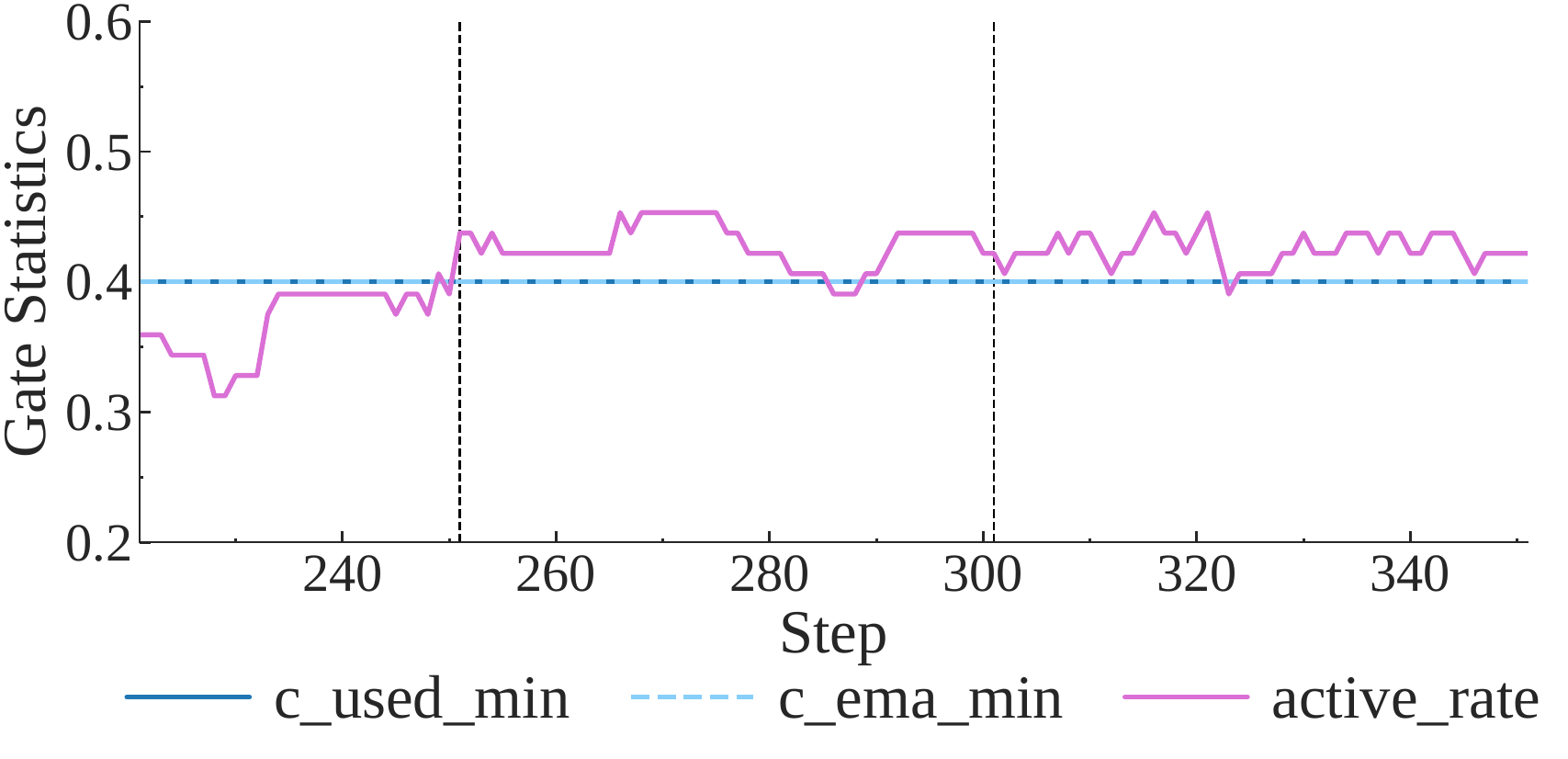}
  \caption{\textbf{Closed-loop clipping evidence.}
  Outlier events increase the column-energy ratio in log-scale (top), followed by stronger \emph{feature-wise clipping} in the applied coefficients (bottom; $c_{\mathrm{used,min}}$).}
  \label{fig:toy2_closed_loop}
\end{figure}

\paragraph{Not a trivial step-size reduction.}
To rule out the confound that improvements arise from a global effective step-size change, we include a \textsc{TrasMuon-noClip} ablation that disables feature-wise clipping while keeping all other components identical.
As summarized in Table~\ref{tab:toy2_fixV_true}, \textsc{TrasMuon-noClip} behaves similarly to NorMuon, whereas enabling clipping yields a clear reduction in spike statistics and a large improvement in the final objective, isolating the contribution of feature-wise clipping.

\begin{table}[t]
\centering
\caption{\textbf{Mechanistic study summary under \texttt{fix\_V=True}.}
\textsc{TrasMuon-noClip} removes feature-wise clipping while keeping the rest identical.
We report median with IQR across runs; lower is better.}
\label{tab:toy2_fixV_true}
\small
\begin{tabular}{lcc}
\toprule
Method & Spike Count & Final Loss \\
\midrule
NorMuon & 44 \tiny{(35,56)}  & 1.3e+06 \tiny{(1.0e+06,1.6e+06)} \\
\textsc{TrasMuon}-noClip & 48 \tiny{(38,56)} & 1.1e+06 \tiny{(8.9e+05,1.9e+06)} \\
\textsc{TrasMuon}-Clip-only & \textbf{28} \tiny{(24,34)} & 2.4e+05 \tiny{(2.0e+05,2.8e+05)} \\
\textsc{TrasMuon}-Clip-sf & 30 \tiny{(24,36)} & \textbf{2.0e+05} \tiny{(1.6e+05,2.7e+05)} \\
\bottomrule
\end{tabular}
\end{table}

\paragraph{Boundary condition (feature semantics broken).}
When the column basis is randomized, the advantage of feature-wise clipping diminishes (Appendix~\ref{app:toy2_supp}), consistent with the intended mechanism: axis-aligned clipping requires a meaningful feature basis.

\section{Discussion and Limitations}
\label{sec:discussion}

\paragraph{What \textsc{TrasMuon} changes.}
\textsc{TrasMuon} factorizes matrix updates into (i) a Muon-style near-isometric mixing factor constructed by Newton--Schulz orthogonalization and (ii) explicit magnitude controls: a global RMS-calibrated step size and a bounded, damping-only feature-wise coefficient $c_{t,j}\in[c_{\min},1]$.
This design targets a common practical tension: structured mixing can improve optimization geometry, while stable magnitudes govern learning-rate sensitivity and robustness to heavy-tailed bursts.
When feature axes are semantically meaningful and bursts are axis-localized, relative-energy damping selectively attenuates high-energy columns while largely preserving the Muon-style mixing structure.

\paragraph{When feature-wise damping and effective-time smoothing help.}
Feature-wise damping is most beneficial when update energy concentrates on a small subset of feature axes, reflected by large relative ratios $r_{t,j}=E_{t,j}/(E_t^{\mathrm{ref}}+\epsilon)$.
In this regime, multiplicative damping suppresses burst-dominated columns without amplification.
When clipping is recomputed sparsely (every $K$ steps) or schedules vary, effective-time (schedule-free) averaging provides a stable long-horizon estimate by accumulating $\gamma_t^2$-weighted statistics, reducing sensitivity to recomputation frequency and schedule details.


\paragraph{Limitations.}
(i) The formulation is most natural for 2D weight matrices; extending energy diagnostics and damping to embeddings and higher-order tensors requires careful axis conventions.
(ii) Newton--Schulz orthogonalization is sensitive to numerical precision; large-scale deployment benefits from precision-aware implementations.
(iii) \textsc{TrasMuon} introduces additional hyperparameters and design choices (e.g., $K,\alpha,k,c_{\min},\rho$ and the robust reference), and their interactions with model scale and data regimes merit broader sweeps.

\section{Conclusion}
\label{sec:conclusion}

We presented \textsc{TrasMuon}, a Muon-family optimizer that combines (i) NS-based near-isometric mixing factors with (ii) explicit magnitude stabilization via global RMS calibration and bounded, damping-only feature-wise clipping, optionally smoothed by effective-time (schedule-free) averaging.
Across the evaluated workloads, \textsc{TrasMuon} improves training stability and achieves competitive or better final performance than strong baselines.
Controlled diagnostics further support the intended mechanism: column-localized energy bursts increase relative energy ratios and lead to stronger applied damping, while ablations (e.g., \textsc{noClip}) and broken-axis settings help rule out trivial explanations such as uniform step-size reduction.
On practical tasks including language-model training, vision transformers, and PINNs under ROI-induced sampling shifts, \textsc{TrasMuon} yields faster or more stable optimization dynamics and improved robustness. 

Several extensions remain. Feature-wise damping is currently formulated for 2D matrices; extending it to embeddings and higher-order tensors requires principled axis conventions and efficient implementations.
NS orthogonalization adds compute and can be sensitive to numerical precision, motivating precision-aware and mixed-precision robust variants.
Finally, a fuller theory connecting orthogonalized mixing, axis-selective damping, and generalization remains open (Appendix~\ref{app:future_work}).


\section*{Impact Statement}

This paper presents work whose goal is to advance the field of machine learning.
There are many potential societal consequences of our work, none of which we feel must be specifically highlighted here.

\bibliography{Optimizer_trimmed}

\begin{thebibliography}{51}
\providecommand{\natexlab}[1]{#1}
\providecommand{\url}[1]{\texttt{#1}}
\expandafter\ifx\csname urlstyle\endcsname\relax
  \providecommand{\doi}[1]{doi: #1}\else
  \providecommand{\doi}{doi: \begingroup \urlstyle{rm}\Url}\fi

\bibitem[Ahn et~al.(2025)Ahn, Xu, Abreu, Fan, Magakyan, Sharma, Zhan, and Langford]{ahn2025dion}
Ahn, K., Xu, B., Abreu, N., Fan, Y., Magakyan, G., Sharma, P., Zhan, Z., and Langford, J.
\newblock Dion: Distributed orthonormalized updates.
\newblock \emph{arXiv preprint arXiv:2504.05295}, 2025.

\bibitem[Allouah et~al.(2025)Allouah, Guerraoui, Gupta, Jellouli, Rizk, and Stephan]{allouah2025adaptiveclippingFL}
Allouah, Y., Guerraoui, R., Gupta, N., Jellouli, A., Rizk, G., and Stephan, J.
\newblock Adaptive gradient clipping for robust federated learning.
\newblock In \emph{International Conference on Learning Representations}, 2025.

\bibitem[Ambityga(2021)]{ambityga2021imagenet100}
Ambityga.
\newblock {ImageNet100}.
\newblock \url{https://www.kaggle.com/datasets/ambityga/imagenet100}, 2021.
\newblock Accessed: 2026-01-23.

\bibitem[{Anthropic}(2026)]{IntroducingClaudeOpus}
{Anthropic}.
\newblock Introducing {Claude Opus} 4.5, 2026.
\newblock URL \url{https://www.anthropic.com/news/claude-opus-4-5}.
\newblock Accessed: 2026-01-12.

\bibitem[Behrouz et~al.(2025)Behrouz, Li, Kacham, Daliri, Deng, Zhong, Razaviyayn, and Mirrokni]{behrouzATLASLearningOptimally2025}
Behrouz, A., Li, Z., Kacham, P., Daliri, M., Deng, Y., Zhong, P., Razaviyayn, M., and Mirrokni, V.
\newblock {ATLAS}: Learning to optimally memorize the context at test time.
\newblock \emph{arXiv preprint arXiv:2505.23735}, 2025.

\bibitem[Bernstein(2025)]{bernstein2025deriving}
Bernstein, J.
\newblock Deriving {Muon}, 2025.
\newblock URL \url{https://jeremybernste.in/writing/deriving-muon}.

\bibitem[Bernstein \& Newhouse(2024)Bernstein and Newhouse]{bernsteinOldOptimizerNew2024a}
Bernstein, J. and Newhouse, L.
\newblock Old optimizer, new norm: An anthology.
\newblock In \emph{OPT 2024: Optimization for Machine Learning}, 2024.

\bibitem[Bernstein \& Newhouse(2025)Bernstein and Newhouse]{bernsteinModularDualityDeep2024a}
Bernstein, J. and Newhouse, L.
\newblock Modular duality in deep learning.
\newblock In \emph{International Conference on Machine Learning}, 2025.

\bibitem[Brock et~al.(2021)Brock, De, Smith, and Simonyan]{brock2021agc}
Brock, A., De, S., Smith, S.~L., and Simonyan, K.
\newblock High-performance large-scale image recognition without normalization.
\newblock In \emph{International Conference on Machine Learning}, 2021.

\bibitem[Conn et~al.(2000)Conn, Gould, and Toint]{conn2000trust}
Conn, A.~R., Gould, N. I.~M., and Toint, P.~L.
\newblock \emph{Trust Region Methods}.
\newblock SIAM, 2000.

\bibitem[{DeepSeek-AI} et~al.(2025){DeepSeek-AI}, Liu, Feng, Xue, Wang, et~al.]{deepseek-aiDeepSeekV3TechnicalReport2025}
{DeepSeek-AI}, Liu, A., Feng, B., Xue, B., Wang, B., et~al.
\newblock {DeepSeek-V3} technical report.
\newblock \emph{arXiv preprint arXiv:2412.19437}, 2025.

\bibitem[Defazio et~al.(2024)Defazio, Yang, Khaled, Mishchenko, Mehta, and Cutkosky]{defazio2024road}
Defazio, A., Yang, X.~A., Khaled, A., Mishchenko, K., Mehta, H., and Cutkosky, A.
\newblock The road less scheduled.
\newblock In \emph{Advances in Neural Information Processing Systems}, 2024.

\bibitem[Deng et~al.(2009)Deng, Dong, Socher, Li, Li, and Fei-Fei]{deng2009imagenet}
Deng, J., Dong, W., Socher, R., Li, L.-J., Li, K., and Fei-Fei, L.
\newblock Imagenet: A large-scale hierarchical image database.
\newblock In \emph{IEEE Conference on Computer Vision and Pattern Recognition}, 2009.

\bibitem[Dosovitskiy et~al.(2021)Dosovitskiy, Beyer, Kolesnikov, Weissenborn, Zhai, Unterthiner, Dehghani, Minderer, Heigold, Gelly, Uszkoreit, and Houlsby]{dosovitskiy2020vit}
Dosovitskiy, A., Beyer, L., Kolesnikov, A., Weissenborn, D., Zhai, X., Unterthiner, T., Dehghani, M., Minderer, M., Heigold, G., Gelly, S., Uszkoreit, J., and Houlsby, N.
\newblock An image is worth 16x16 words: Transformers for image recognition at scale.
\newblock In \emph{International Conference on Learning Representations}, 2021.

\bibitem[Duchi et~al.(2011)Duchi, Hazan, and Singer]{duchi2011adagrad}
Duchi, J., Hazan, E., and Singer, Y.
\newblock Adaptive subgradient methods for online learning and stochastic optimization.
\newblock \emph{Journal of Machine Learning Research}, 12:\penalty0 2121--2159, 2011.

\bibitem[Gao et~al.(2023)Gao, Yan, and Zhou]{gao2023failure}
Gao, Z., Yan, L., and Zhou, T.
\newblock Failure-informed adaptive sampling for {PINNs}.
\newblock \emph{SIAM Journal on Scientific Computing}, 45\penalty0 (4):\penalty0 A1971--A1994, 2023.

\bibitem[Gupta \& Wojtowytsch(2025)Gupta and Wojtowytsch]{gupta2024nesterov}
Gupta, K. and Wojtowytsch, S.
\newblock Nesterov acceleration in benignly non-convex landscapes.
\newblock In \emph{The International Conference on Learning Representations}, 2025.

\bibitem[Gupta et~al.(2018)Gupta, Koren, and Singer]{gupta2018shampoo}
Gupta, V., Koren, T., and Singer, Y.
\newblock Shampoo: Preconditioned stochastic tensor optimization.
\newblock In \emph{International Conference on Machine Learning}, 2018.

\bibitem[Hampel et~al.(1986)Hampel, Ronchetti, Rousseeuw, and Stahel]{hampel1986robust}
Hampel, F.~R., Ronchetti, E.~M., Rousseeuw, P.~J., and Stahel, W.~A.
\newblock \emph{Robust Statistics: The Approach Based on Influence Functions}.
\newblock John Wiley \& Sons, 1986.

\bibitem[Huber(1981)]{huber1981robust}
Huber, P.~J.
\newblock \emph{Robust Statistics}.
\newblock John Wiley \& Sons, 1981.

\bibitem[Izmailov et~al.(2018)Izmailov, Wilson, Podoprikhin, Vetrov, and Garipov]{izmailov2018swa}
Izmailov, P., Wilson, A., Podoprikhin, D., Vetrov, D., and Garipov, T.
\newblock Averaging weights leads to wider optima and better generalization.
\newblock In \emph{Conference on Uncertainty in Artificial Intelligence}, 2018.

\bibitem[Jordan et~al.(2024)Jordan, Jin, Boza, You, Cesista, Newhouse, and Bernstein]{jordan2024muon}
Jordan, K., Jin, Y., Boza, V., You, J., Cesista, F., Newhouse, L., and Bernstein, J.
\newblock {Muon}: An optimizer for hidden layers in neural networks, 2024.
\newblock URL \url{https://kellerjordan.github.io/posts/muon/}.

\bibitem[Khaled et~al.(2025)Khaled, Ozkara, Yu, Hong, and Park]{khaled2025muonbp}
Khaled, A., Ozkara, K., Yu, T., Hong, M., and Park, Y.
\newblock {MuonBP}: Faster {Muon} via block-periodic orthogonalization.
\newblock \emph{arXiv preprint arXiv:2510.16981}, 2025.

\bibitem[{Kimi Team} et~al.(2025){Kimi Team}, Bai, Bao, Chen, Chen, Chen, Chen, et~al.]{teamKimiK2Open2025a}
{Kimi Team}, Bai, Y., Bao, Y., Chen, G., Chen, J., Chen, N., Chen, R., et~al.
\newblock {Kimi K2}: Open agentic intelligence.
\newblock \emph{arXiv preprint arXiv:2507.20534}, 2025.

\bibitem[Kingma \& Ba(2017)Kingma and Ba]{kingma2017adammethodstochasticoptimization}
Kingma, D.~P. and Ba, J.
\newblock {Adam}: A method for stochastic optimization.
\newblock \emph{arXiv preprint arXiv:1412.6980}, 2017.

\bibitem[Kumar et~al.(2025)Kumar, Jin, and Quesnelle]{kumar2025curvadion}
Kumar, B., Jin, R., and Quesnelle, J.
\newblock {CurvaDion}: Curvature-adaptive distributed orthonormalization.
\newblock \emph{arXiv preprint arXiv:2512.13728}, 2025.

\bibitem[Large et~al.(2024)Large, Liu, Huh, Bahng, Isola, and Bernstein]{largeScalableOptimizationModular2024}
Large, T., Liu, Y., Huh, M., Bahng, H., Isola, P., and Bernstein, J.
\newblock Scalable optimization in the modular norm.
\newblock In \emph{Advances in Neural Information Processing Systems}, 2024.

\bibitem[Li et~al.(2025)Li, Liu, Liang, Chen, and Zhao]{liNorMuonMakingMuon2025a}
Li, Z., Liu, L., Liang, C., Chen, W., and Zhao, T.
\newblock {NorMuon}: Making {Muon} more efficient and scalable.
\newblock \emph{arXiv preprint arXiv:2510.05491}, 2025.

\bibitem[Loshchilov \& Hutter(2019)Loshchilov and Hutter]{loshchilov2019decoupledweightdecayregularization}
Loshchilov, I. and Hutter, F.
\newblock Decoupled weight decay regularization.
\newblock In \emph{International Conference on Learning Representations}, 2019.

\bibitem[Lozhkov et~al.(2024)Lozhkov, Ben~Allal, von Werra, and Wolf]{fineweb-edu}
Lozhkov, A., Ben~Allal, L., von Werra, L., and Wolf, T.
\newblock {FineWeb-Edu}: The finest collection of educational content, 2024.
\newblock URL \url{https://huggingface.co/datasets/HuggingFaceFW/fineweb-edu}.
\newblock Accessed: 2026-01-01.

\bibitem[Marfinetz(2025)]{marfinetzEvolvingDeepLearning2025}
Marfinetz, M.
\newblock Evolving {{Deep Learning Optimizers}}.
\newblock \emph{arXiv preprint arXiv:2512.11853}, 2025.

\bibitem[Martens \& Grosse(2015)Martens and Grosse]{martens2015kfac}
Martens, J. and Grosse, R.
\newblock Optimizing neural networks with kronecker-factored approximate curvature.
\newblock In \emph{International Conference on Machine Learning}, 2015.

\bibitem[{OpenAI}(2026)]{IntroducingGPT522026}
{OpenAI}.
\newblock Introducing {{GPT-5.2}}, 2026.
\newblock URL \url{https://openai.com/index/introducing-gpt-5-2/}.
\newblock Accessed: 2026-01-12.

\bibitem[Pagliardini et~al.(2025)Pagliardini, Ablin, and Grangier]{pagliardini2024ademamixoptimizerbetterfaster}
Pagliardini, M., Ablin, P., and Grangier, D.
\newblock The {AdEMAMix} optimizer: Better, faster, older.
\newblock In \emph{International Conference on Learning Representations}, 2025.

\bibitem[Park et~al.(2025)Park, Lee, Yoon, Hwang, and Kang]{parkOutlierSafePreTrainingRobust2025}
Park, J., Lee, T., Yoon, C., Hwang, H., and Kang, J.
\newblock Outlier-safe pre-training for robust {4-Bit} quantization of large language models.
\newblock \emph{arXiv preprint arXiv:2506.19697}, 2025.

\bibitem[Pascanu et~al.(2013)Pascanu, Mikolov, and Bengio]{pascanu2013difficulty}
Pascanu, R., Mikolov, T., and Bengio, Y.
\newblock On the difficulty of training recurrent neural networks.
\newblock In \emph{International Conference on Machine Learning}, 2013.

\bibitem[Pethick et~al.(2025)Pethick, Xie, Antonakopoulos, Zhu, Silveti-Falls, and Cevher]{pethick2025training}
Pethick, T., Xie, W., Antonakopoulos, K., Zhu, Z., Silveti-Falls, A., and Cevher, V.
\newblock Training deep learning models with norm-constrained {LMO}s.
\newblock In \emph{International Conference on Machine Learning}, 2025.

\bibitem[Polyak \& Juditsky(1992)Polyak and Juditsky]{polyak1992averaging}
Polyak, B.~T. and Juditsky, A.~B.
\newblock Acceleration of stochastic approximation by averaging.
\newblock \emph{SIAM Journal on Control and Optimization}, 1992.

\bibitem[Radford et~al.(2019)Radford, Wu, Child, Luan, Amodei, Sutskever, et~al.]{gpt2}
Radford, A., Wu, J., Child, R., Luan, D., Amodei, D., Sutskever, I., et~al.
\newblock Language models are unsupervised multitask learners.
\newblock \emph{OpenAI blog}, 2019.

\bibitem[Riabinin et~al.(2025)Riabinin, Shulgin, Gruntkowska, and Richt{\'a}rik]{riabinin2025gluon}
Riabinin, A., Shulgin, E., Gruntkowska, K., and Richt{\'a}rik, P.
\newblock Gluon: Making {Muon} \& {Scion} great again! (bridging theory and practice of {LMO}-based optimizers for {LLMs}).
\newblock \emph{arXiv preprint arXiv:2505.13416}, 2025.

\bibitem[Ruppert(1988)]{ruppert1988rm}
Ruppert, D.
\newblock Efficient estimations from a slowly convergent robbins--monro process.
\newblock Technical report, Cornell University, ORIE, 1988.

\bibitem[Shao et~al.(2025)Shao, Weng, Sun, Gao, Zhang, Mao, Xu, Zhang, and Xing]{shao2025bds}
Shao, Y., Weng, S., Sun, H., Gao, Q., Zhang, L., Mao, Z., Xu, S., Zhang, Z., and Xing, L.
\newblock {BDS-Adam} optimizer integrating adaptive variance rectification with semi-adaptive gradient smoothing.
\newblock \emph{Scientific Reports}, 15\penalty0 (1):\penalty0 36906, 2025.

\bibitem[Shazeer \& Stern(2018)Shazeer and Stern]{shazeer2018adafactor}
Shazeer, N. and Stern, M.
\newblock Adafactor: Adaptive learning rates with sublinear memory cost.
\newblock In \emph{International Conference on Machine Learning}, 2018.

\bibitem[Subramanian et~al.(2022)Subramanian, Kirby, Mahoney, and Gholami]{subramanian2022adaptive}
Subramanian, S., Kirby, R.~M., Mahoney, M.~W., and Gholami, A.
\newblock Adaptive self-supervision algorithms for physics-informed neural networks.
\newblock \emph{arXiv preprint arXiv:2207.04084}, 2022.

\bibitem[Touvron et~al.(2021)Touvron, Cord, Douze, Massa, Sablayrolles, and J{\'e}gou]{touvron2021training}
Touvron, H., Cord, M., Douze, M., Massa, F., Sablayrolles, A., and J{\'e}gou, H.
\newblock Training data-efficient image transformers \& distillation through attention.
\newblock In \emph{International Conference on Machine Learning}, 2021.

\bibitem[Wu et~al.(2023)Wu, Zhu, Tan, Kartha, and Lu]{wu2023comprehensive}
Wu, C., Zhu, M., Tan, Q., Kartha, Y., and Lu, L.
\newblock A comprehensive study of non-adaptive and residual-based adaptive sampling for physics-informed neural networks.
\newblock \emph{Computer Methods in Applied Mechanics and Engineering}, 403:\penalty0 115671, 2023.

\bibitem[Yang et~al.(2025)Yang, Li, Yang, Zhang, Hui, Zheng, Yu, Gao, Huang, Lv, et~al.]{yangQwen3TechnicalReport2025}
Yang, A., Li, A., Yang, B., Zhang, B., Hui, B., Zheng, B., Yu, B., Gao, C., Huang, C., Lv, C., et~al.
\newblock {Qwen3} technical report.
\newblock \emph{arXiv preprint arXiv:2505.09388}, 2025.

\bibitem[You et~al.(2017)You, Gitman, and Ginsburg]{you2017lars}
You, Y., Gitman, I., and Ginsburg, B.
\newblock Large batch training of convolutional networks.
\newblock \emph{arXiv preprint arXiv:1708.03888}, 2017.

\bibitem[You et~al.(2020)You, Li, Reddi, Hseu, Kumar, Bhojanapalli, Song, Demmel, Keutzer, and Hsieh]{you2019lamb}
You, Y., Li, J., Reddi, S., Hseu, J., Kumar, S., Bhojanapalli, S., Song, X., Demmel, J., Keutzer, K., and Hsieh, C.-J.
\newblock Large batch optimization for deep learning: Training bert in 76 minutes.
\newblock In \emph{International Conference on Learning Representations}, 2020.

\bibitem[Yuan et~al.(2025)Yuan, Liu, Wu, Zhou, and Gu]{yuan2024mars}
Yuan, H., Liu, Y., Wu, S., Zhou, X., and Gu, Q.
\newblock {MARS}: Unleashing the power of variance reduction for training large models.
\newblock In \emph{International Conference on Machine Learning}, 2025.

\bibitem[Zhang et~al.(2019)Zhang, Lucas, Ba, and Hinton]{zhang2019lookahead}
Zhang, M., Lucas, J., Ba, J., and Hinton, G.~E.
\newblock Lookahead optimizer: k steps forward, 1 step back.
\newblock \emph{Advances in Neural Information Processing Systems}, 2019.

\end{thebibliography}


\begin{thebibliography}{8}
\providecommand{\natexlab}[1]{#1}
\providecommand{\url}[1]{\texttt{#1}}
\expandafter\ifx\csname urlstyle\endcsname\relax
  \providecommand{\doi}[1]{doi: #1}\else
  \providecommand{\doi}{doi: \begingroup \urlstyle{rm}\Url}\fi

\bibitem[Author(2021)]{anonymous}
Author, N.~N.
\newblock Suppressed for anonymity, 2021.

\bibitem[Duda et~al.(2000)Duda, Hart, and Stork]{DudaHart2nd}
Duda, R.~O., Hart, P.~E., and Stork, D.~G.
\newblock \emph{Pattern Classification}.
\newblock John Wiley and Sons, 2nd edition, 2000.

\bibitem[Kearns(1989)]{kearns89}
Kearns, M.~J.
\newblock \emph{Computational Complexity of Machine Learning}.
\newblock PhD thesis, Department of Computer Science, Harvard University, 1989.

\bibitem[Langley(2000)]{langley00}
Langley, P.
\newblock Crafting papers on machine learning.
\newblock In Langley, P. (ed.), \emph{Proceedings of the 17th International Conference on Machine Learning (ICML 2000)}, pp.\  1207--1216, Stanford, CA, 2000. Morgan Kaufmann.

\bibitem[Michalski et~al.(1983)Michalski, Carbonell, and Mitchell]{MachineLearningI}
Michalski, R.~S., Carbonell, J.~G., and Mitchell, T.~M. (eds.).
\newblock \emph{Machine Learning: An Artificial Intelligence Approach, Vol. I}.
\newblock Tioga, Palo Alto, CA, 1983.

\bibitem[Mitchell(1980)]{mitchell80}
Mitchell, T.~M.
\newblock The need for biases in learning generalizations.
\newblock Technical report, Computer Science Department, Rutgers University, New Brunswick, MA, 1980.

\bibitem[Newell \& Rosenbloom(1981)Newell and Rosenbloom]{Newell81}
Newell, A. and Rosenbloom, P.~S.
\newblock Mechanisms of skill acquisition and the law of practice.
\newblock In Anderson, J.~R. (ed.), \emph{Cognitive Skills and Their Acquisition}, chapter~1, pp.\  1--51. Lawrence Erlbaum Associates, Inc., Hillsdale, NJ, 1981.

\bibitem[Samuel(1959)]{Samuel59}
Samuel, A.~L.
\newblock Some studies in machine learning using the game of checkers.
\newblock \emph{IBM Journal of Research and Development}, 3\penalty0 (3):\penalty0 211--229, 1959.

\end{thebibliography}
\bibliographystyle{icml2026}

\newpage
\appendix
\onecolumn

\section{Future Work}
\label{app:future_work}

\begin{itemize}
  \item \textbf{Theory and guarantees.}
  Develop stability and convergence analyses for orthogonalized updates combined with bounded, damping-only feature-wise clipping.
  A promising direction is to formalize a Lyapunov-style descent argument under explicit alignment conditions and bounded-update properties, and to connect empirical spectral diagnostics (e.g., the effective update statistic $A_{\mathrm{eff}}=W^\top\Delta W$) to provable stability regimes.

  \item \textbf{Generalizing feature axes beyond 2D matrices.}
  Extend EnergyCol-style clipping to convolutional kernels, embeddings, and higher-order tensors by defining principled ``feature axes'' (e.g., input-channel, attention-head, or group dimensions).
  We also plan to explore block-wise and low-rank variants that preserve interpretability while reducing per-step overhead.

  \item \textbf{Adaptive burst modeling with transparent control.}
  Replace fixed clipping hyperparameters (e.g., $c_{\min}$, $\alpha$, update period $K$) with lightweight, interpretable adaptations driven by online tail statistics of the energy distribution (quantiles, kurtosis, robust outlier scores), while maintaining the damping-only constraint and avoiding hidden amplification.

  \item \textbf{Systems and numerical precision.}
  Improve the efficiency and robustness of Newton--Schulz orthogonalization in mixed precision and distributed settings.
  This includes precision-aware kernels, communication-efficient implementations, and amortized/approximate orthogonalization strategies that retain most of the directional benefit at lower cost.

  \item \textbf{Broader robustness regimes and downstream impact.}
  Evaluate \textsc{TrasMuon} under realistic forms of nonstationarity common in large-scale training (curriculum shifts, domain-mixture changes, sequence-length spikes, and data-quality transitions), and study how clipping statistics correlate with downstream robustness and generalization.
\end{itemize}

\section{TrasMuon Algorithm (Extended)}
\label{app:methodology}

We present \textbf{TrasMuon} (\textbf{T}rust \textbf{R}egion \textbf{A}daptive \textbf{S}caling for \textbf{Muon}),
which explicitly decouples \emph{update geometry} (direction) from \emph{step-size control} (magnitude).
For a matrix parameter $W\in\mathbb{R}^{d_{\mathrm{out}}\times d_{\mathrm{in}}}$ with stochastic gradient
$G_t=\nabla_W\mathcal{L}(W_t)$, TrasMuon applies multiplicative coupling
\begin{equation}
\Delta W_t \;=\; -\,\hat{\eta}_t \, O_t^{\mathrm{base}} \,\mathrm{diag}(c_t),
\label{eq:trasmuon_update}
\end{equation}
where $O_t^{\mathrm{base}}\in\mathbb{R}^{d_{\mathrm{out}}\times d_{\mathrm{in}}}$ provides a structured direction,
$\hat{\eta}_t$ is a row-wise RMS-calibrated global step size, and
$c_t\in[c_{\min},1]^{d_{\mathrm{in}}}$ is a \emph{feature-axis} (column-wise) damping vector (damping-only, no amplification).

\subsection{Orthogonalized Directions via Newton--Schulz}
\label{sec:direction}
TrasMuon maintains an exponential moving average of gradients
\begin{equation}
M_t \;=\; \beta_1 M_{t-1} + (1-\beta_1)G_t .
\label{eq:momentum}
\end{equation}
To extract a rotation-robust, near-isometric direction, we approximate the polar factor of $M_t$.
For numerical stability of Newton--Schulz (NS) iterations, we remove the scale gauge by RMS-normalizing
\begin{equation}
\tilde{M}_t \;=\; \frac{M_t}{\|M_t\|_F/\sqrt{d_{\mathrm{out}}d_{\mathrm{in}}}+\epsilon},
\label{eq:ns_prenorm}
\end{equation}
and apply $T$ NS steps to obtain
\begin{equation}
O_t \;\approx\; \mathrm{NS}(\tilde{M}_t;T) \;\approx\; \tilde{M}_t(\tilde{M}_t^\top \tilde{M}_t)^{-1/2},
\label{eq:ns_polar}
\end{equation}
yielding a structured direction that is less sensitive to axis rotations than elementwise or diagonal preconditioning
(e.g., Muon-style orthogonalized updates).

\subsection{Row-Second-Moment Scaling and RMS-Calibrated Step Size}
\label{sec:rms}

Orthogonalization primarily shapes \emph{direction}. To stabilize \emph{magnitude} across layers and time, we apply
lightweight row-wise second-moment scaling (as in NorMuon~\cite{liNorMuonMakingMuon2025a}):
\begin{equation}
v_t^{\mathrm{row}} \;=\; \beta_2 v_{t-1}^{\mathrm{row}} + (1-\beta_2)\,\mathrm{mean}_{j}\!\left(O_{t,\cdot j}^{\odot 2}\right),
\label{eq:row_scale1}
\end{equation}
\begin{equation}
O_t^{\mathrm{base}} \;=\; \mathrm{diag}\!\big((v_t^{\mathrm{row}}+\epsilon)^{-1/2}\big)\,O_t .
\label{eq:row_scale2}
\end{equation}
Row scaling addresses row-wise heterogeneity, while a \emph{global} row-wise calibration controls the update norm.
We set
\begin{equation}
\hat{\eta}_t
\;=\;
\eta \cdot \frac{\sqrt{d_{\mathrm{out}}d_{\mathrm{in}}}}{\|O_t^{\mathrm{base}}\|_F+\epsilon},
\label{eq:rms_cal}
\end{equation}
so that the per-step RMS magnitude of $\Delta W_t$ is on the order of $\eta$.
Since $c_t\le \mathbf{1}$ elementwise, \eqref{eq:rms_cal} also implies an explicit Frobenius-norm bound
$\|\Delta W_t\|_F \le \eta\sqrt{d_{\mathrm{out}}d_{\mathrm{in}}}$ (up to $\epsilon$), reducing sensitivity to layer shape
and fluctuations in the orthogonalized direction~\cite{bernsteinModularDualityDeep2024a,largeScalableOptimizationModular2024}.

\subsection{Energy-Based Feature Clipping}
\label{sec:energy_gate}

\paragraph{Motivation.}
In practice, instability often arises from \emph{bursty magnitudes} that concentrate on a small subset of feature axes (columns),
causing loss spikes and narrowing the stable learning-rate region.
TrasMuon therefore introduces \emph{feature-wise clipping}: it dampens only the high-energy feature directions while preserving the Muon-like direction structure in $O_t^{\mathrm{base}}$.

\paragraph{Column energy and a robust reference.}
We measure column energy from momentum $M_t$:
\begin{equation}
E_{t,j} \;=\; \sum_{i=1}^{d_{\mathrm{out}}} M_{t,ij}^2,
\qquad j=1,\dots,d_{\mathrm{in}}.
\label{eq:col_energy}
\end{equation}
We summarize the typical energy level at step $t$ by a quantile statistic
\begin{equation}
E_t^{\mathrm{cur}}
\;=\;
\mathrm{Quantile}_{q}\!\big(\{E_{t,j}\}_{j=1}^{d_{\mathrm{in}}}\big),
\qquad q=0.5,
\label{eq:ecur_quantile}
\end{equation}
and maintain a running reference via an EMA updated every step:
\begin{equation}
E_t^{\mathrm{ref}}
\;=\;
\beta_E E_{t-1}^{\mathrm{ref}} + (1-\beta_E)\,E_t^{\mathrm{cur}}.
\label{eq:eref}
\end{equation}
Using a quantile (median) yields a high-breakdown reference: the sample median has a $50\%$ breakdown point,
so a sparse set of high-energy columns cannot arbitrarily inflate $E_t^{\mathrm{ref}}$ and thereby ``move the clipping threshold''
in response to the outliers being clipped~\cite{hampel1986robust,huber1981robust}.

\paragraph{Relative ratio and clipping-style damping.}
We form a dimensionless ratio
\begin{equation}
r_{t,j} \;=\; \frac{E_{t,j}}{E_t^{\mathrm{ref}}+\epsilon}.
\label{eq:ratio}
\end{equation}
A hard energy cap $E_{t,j}\le k\,E_t^{\mathrm{ref}}$ gives the column-wise analogue of norm clipping:
\begin{equation}
c_{t,j}^{\mathrm{hard}}
\;=\;
\min\!\left(1,\sqrt{\frac{k\,E_t^{\mathrm{ref}}}{E_{t,j}+\epsilon}}\right),
\label{eq:hard_col_clip}
\end{equation}
which enforces $r_{t,j}\le k$ after rescaling.
In TrasMuon we use a smooth, numerically stable \emph{soft clipping} rule:
\begin{equation}
c_{t,j}^{\mathrm{raw}}
\;=\;
\frac{1}{1+\alpha\log(1+r_{t,j})},
\qquad
c_{t,j}^{\mathrm{gate}}
\;=\;
\mathrm{clip}\!\left(c_{t,j}^{\mathrm{raw}},\,c_{\min},\,1\right).
\label{eq:soft_clip}
\end{equation}
This rule is bounded and avoids power-law instabilities as $r_{t,j}\!\to\!0$.
Importantly, $c_{t,j}\le 1$ for all $j$, so the mechanism is strictly damping-only and can be interpreted as a trust-region safety mechanism in \emph{feature space}.

\paragraph{Triggered vs.\ continuous clipping.}
Optionally, we apply damping only when $r_{t,j}$ exceeds a trigger $k$:
\begin{equation}
\bar{c}_{t,j}
\;=\;
\begin{cases}
c_{t,j}^{\mathrm{gate}}, & r_{t,j} > k,\\
1, & \text{otherwise},
\end{cases}
\label{eq:trigger_gate}
\end{equation}
so that non-burst columns remain unchanged and the gate acts as an event-driven clip.

\subsection{Temporal Smoothing and Schedule-Free Averaging}
\label{sec:sf}

To reduce short-term noise and avoid sensitivity to the gate-update period, we smooth the instantaneous clip.
Let $c_t^{\mathrm{inst}}$ denote the clip applied at step $t$ (either $c_t^{\mathrm{gate}}$ or $\bar{c}_t$ depending on triggering).
We first apply EMA smoothing:
\begin{equation}
c_t^{\mathrm{ema}}
\;=\;
\beta_c c_{t-1}^{\mathrm{ema}} + (1-\beta_c)c_t^{\mathrm{inst}}.
\label{eq:gate_ema}
\end{equation}
Second, we maintain a schedule-free average using an effective step weight $\gamma_t$ (default $\gamma_t=\eta$)~\cite{defazio2024road}.
Define the scalar accumulator $S_t\in\mathbb{R}$ and vector accumulator $C_t\in\mathbb{R}^{d_{\mathrm{in}}}$:
\begin{equation}
S_t \;=\; S_{t-1}+\gamma_t^2,
\qquad
C_t \;=\; C_{t-1}+\gamma_t^2 c_{t-1}^{\mathrm{last}},
\qquad
c_t^{\mathrm{avg}} \;=\; \frac{C_t}{S_t+\epsilon},
\label{eq:sf_gate}
\end{equation}
where $c_{t-1}^{\mathrm{last}}$ is the most recently applied clip (cached between updates).
We then mix short- and long-term estimates:
\begin{equation}
c_t
\;=\;
(1-\rho)c_t^{\mathrm{ema}} + \rho c_t^{\mathrm{avg}},
\qquad
c_t^{\mathrm{last}} \leftarrow c_t.
\label{eq:mix_gate}
\end{equation}
This effective-time averaging reduces sensitivity to warmup length and total training steps, and prevents bias when the raw clip is computed only every $K$ steps.

\subsection{Final TrasMuon update}
Substituting the components into \eqref{eq:trasmuon_update} yields the final update
\begin{equation}
\Delta W_t
\;=\;
-\hat{\eta}_t\, O_t^{\mathrm{base}}\,\mathrm{diag}(c_t),
\label{eq:final_update_gate}
\end{equation}
which preserves Muon-like directional geometry via $O_t^{\mathrm{base}}$ while controlling bursty magnitudes along feature axes through damping-only feature clipping $c_t$.

\section{Convergence Analysis (Extended)}
\label{app:convergence}

\paragraph{Scope.}We provide a convergence \emph{framework} for TrasMuon/energy clipping updates,
separating (i) unconditional algebraic properties (bounded update norm; damping-only contraction)
from (ii) mild alignment assumptions connecting the structured update to descent.

\subsection{Algorithmic abstraction}
Let $f:\mathbb{R}^{m\times n}\to\mathbb{R}$ and let $G_t$ be a stochastic gradient at $W_t$.
We study updates of the form
\begin{equation}
W_{t+1} \;=\; (1-\eta\lambda)\,W_t \;+\; \Delta W_t,
\qquad
\Delta W_t \;:=\; -\hat{\eta}_t U_t,
\qquad
U_t := O_t^{\mathrm{base}}\mathrm{diag}(c_t),
\label{eq:unified_update_app}
\end{equation}
with RMS-calibrated step size
\begin{equation}
\hat{\eta}_t \;=\; \eta\cdot\frac{\sqrt{mn}}{\|O_t^{\mathrm{base}}\|_F+\epsilon},
\label{eq:rms_eta_app}
\end{equation}
and damping-only clipping
\begin{equation}
0<c_{\min}\le c_{t,j}\le 1 \quad \forall j,t.
\label{eq:damping_only_app}
\end{equation}

\subsection{Deterministic algebraic properties}

\begin{lemma}[Damping-only contraction]
\label{lem:damping_contraction_app}
For any $A\in\mathbb{R}^{m\times n}$ and any $c\in[0,1]^n$,
$\|A\,\mathrm{diag}(c)\|_F \le \|A\|_F$.
\end{lemma}
\begin{proof}
$\|A\,\mathrm{diag}(c)\|_F^2=\sum_{j=1}^n c_j^2\|A_{\cdot j}\|_2^2\le\sum_{j=1}^n\|A_{\cdot j}\|_2^2=\|A\|_F^2$.
\end{proof}

\begin{lemma}[Row-wise RMS calibration]
\label{lem:tr_bound_app}
Under \eqref{eq:unified_update_app}--\eqref{eq:damping_only_app},
for all $t$,
\begin{equation}
\|\Delta W_t\|_F \le \eta\sqrt{mn}.
\label{eq:tr_bound_app}
\end{equation}
\end{lemma}
\begin{proof}
By Lemma~\ref{lem:damping_contraction_app}, $\|U_t\|_F=\|O_t^{\mathrm{base}}\mathrm{diag}(c_t)\|_F\le\|O_t^{\mathrm{base}}\|_F$.
Thus
\[
\|\Delta W_t\|_F=\hat{\eta}_t\|U_t\|_F
\le
\eta\sqrt{mn}\cdot\frac{\|O_t^{\mathrm{base}}\|_F}{\|O_t^{\mathrm{base}}\|_F+\epsilon}
\le \eta\sqrt{mn}.
\]
\end{proof}

\subsection{Assumptions for descent}

\begin{assumption}[$L$-smoothness]
\label{ass:smooth_app}
$f$ is $L$-smooth with respect to Frobenius norm:
$\|\nabla f(X)-\nabla f(Y)\|_F\le L\|X-Y\|_F$.
\end{assumption}

\begin{assumption}[Stochastic gradients]
\label{ass:sgd_app}
$\mathbb{E}[G_t\mid W_t]=\nabla f(W_t)$ and
$\mathbb{E}[\|G_t-\nabla f(W_t)\|_F^2\mid W_t]\le\sigma^2$.
\end{assumption}

\begin{assumption}[Alignment on the realized update]
\label{ass:align_delta_app}
There exists $\mu_\Delta>0$ such that for all $t$,
\begin{equation}
\mathbb{E}\big[\langle \nabla f(W_t), \Delta W_t\rangle \mid W_t\big]
\;\le\;
-\mu_\Delta \,\eta\,\|\nabla f(W_t)\|_F^2,
\label{eq:align_delta_app}
\end{equation}
where $\Delta W_t=-\hat{\eta}_tU_t$.
\end{assumption}

\subsection{Expected stationarity}

\begin{lemma}[Smoothness descent]
\label{lem:smooth_descent_app}
Under Assumption~\ref{ass:smooth_app}, for any $\Delta$,
\begin{equation}
f(W_t+\Delta)\le f(W_t)+\langle\nabla f(W_t),\Delta\rangle+\frac{L}{2}\|\Delta\|_F^2.
\label{eq:smooth_descent_app}
\end{equation}
\end{lemma}

\begin{theorem}[Expected stationarity for RMS-calibrated, damping-only updates]
\label{thm:stationarity_app}
Assume \ref{ass:smooth_app} and \ref{ass:align_delta_app} with $\lambda=0$.
Then for any $T\ge 1$,
\begin{equation}
\frac{1}{T}\sum_{t=0}^{T-1}\mathbb{E}\|\nabla f(W_t)\|_F^2
\;\le\;
\frac{\mathbb{E}[f(W_0)]-f^\star}{\mu_\Delta\,\eta\,T}
\;+\;
\frac{L}{2\mu_\Delta}\,\eta\,mn,
\label{eq:stationarity_rate_app}
\end{equation}
where $f^\star=\inf_W f(W)$.
\end{theorem}

\begin{proof}
Apply Lemma~\ref{lem:smooth_descent_app} with $\Delta=\Delta W_t$, take conditional expectation,
use Assumption~\ref{ass:align_delta_app}, and bound $\|\Delta W_t\|_F^2$ by Lemma~\ref{lem:tr_bound_app}.
Sum over $t=0,\dots,T-1$ and telescope $f(W_t)$.
\end{proof}

\subsection{Extensions: weight decay, stochasticity, and PL}

\paragraph{Weight decay.}
Decoupled weight decay can be handled by analyzing $f_\lambda(W)=f(W)+\frac{\lambda}{2}\|W\|_F^2$
or treating $(1-\eta\lambda)W_t$ as an additional contraction term.

\paragraph{Stochastic gradients.}
Under Assumption~\ref{ass:sgd_app}, the bound acquires an additional $\mathcal{O}(\eta\sigma^2)$ term
as in standard SGD analyses.

\paragraph{PL / strong convexity.}
If $f$ satisfies the PL condition, one obtains linear-type convergence to an $\mathcal{O}(\eta mn)$ (or $\mathcal{O}(\eta\sigma^2)$) neighborhood.


\section{Additional Results: Language-Model Pretraining}
\label{app:LMtraining}

\paragraph{Scope.}
This appendix provides supplementary visualizations for language-model pretraining and a minimal replication check.
The intent is \emph{not} to introduce new claims beyond the main text, but to (i) document late-stage behavior under the fixed training budget and (ii) reduce the risk that qualitative trends are specific to a single random seed.

\paragraph{Fixed-budget protocol (identical across optimizers).}
All runs follow the same fixed-budget protocol as Section~\ref{sec:QWEN}.
Training starts from random initialization and proceeds for 1500 optimization steps on FineWeb-Edu with sequence length 1024 and global batch size 1024, thereby fixing the token budget across optimizers.
A single fixed learning rate $\eta=0.0036$ and weight decay $\lambda=0.005$ are used for all optimizers, with no learning-rate sweep and no optimizer-specific retuning in this appendix.
Two schedule variants are reported: (i) a warmup--stable--decay schedule with 10\% warmup and a final 20\% decay phase, and (ii) the same schedule without warmup.

\paragraph{Metrics.}
Training loss is reported as a high-resolution indicator of early-stage optimization dynamics and stability under an identical training budget.
This appendix presents the corresponding loss trajectories to complement the main-text summaries.

\subsection{Qwen3-0.6B: Late-stage loss under a fixed 1500-step budget}
\label{app:Qwen3_last_step}

The late-stage training window (steps 1200--1500) is examined to complement the early-stage analysis in the main text.
Figure~\ref{fig:qwen_last_steps} shows the corresponding loss trajectories in this window for both schedule variants.
These curves document late-stage behavior under a fixed training budget and are not intended to imply full convergence.

\begin{figure}[t]
  \centering
  \begin{subfigure}[b]{0.49\columnwidth}
    \centering
    \includegraphics[width=\linewidth]{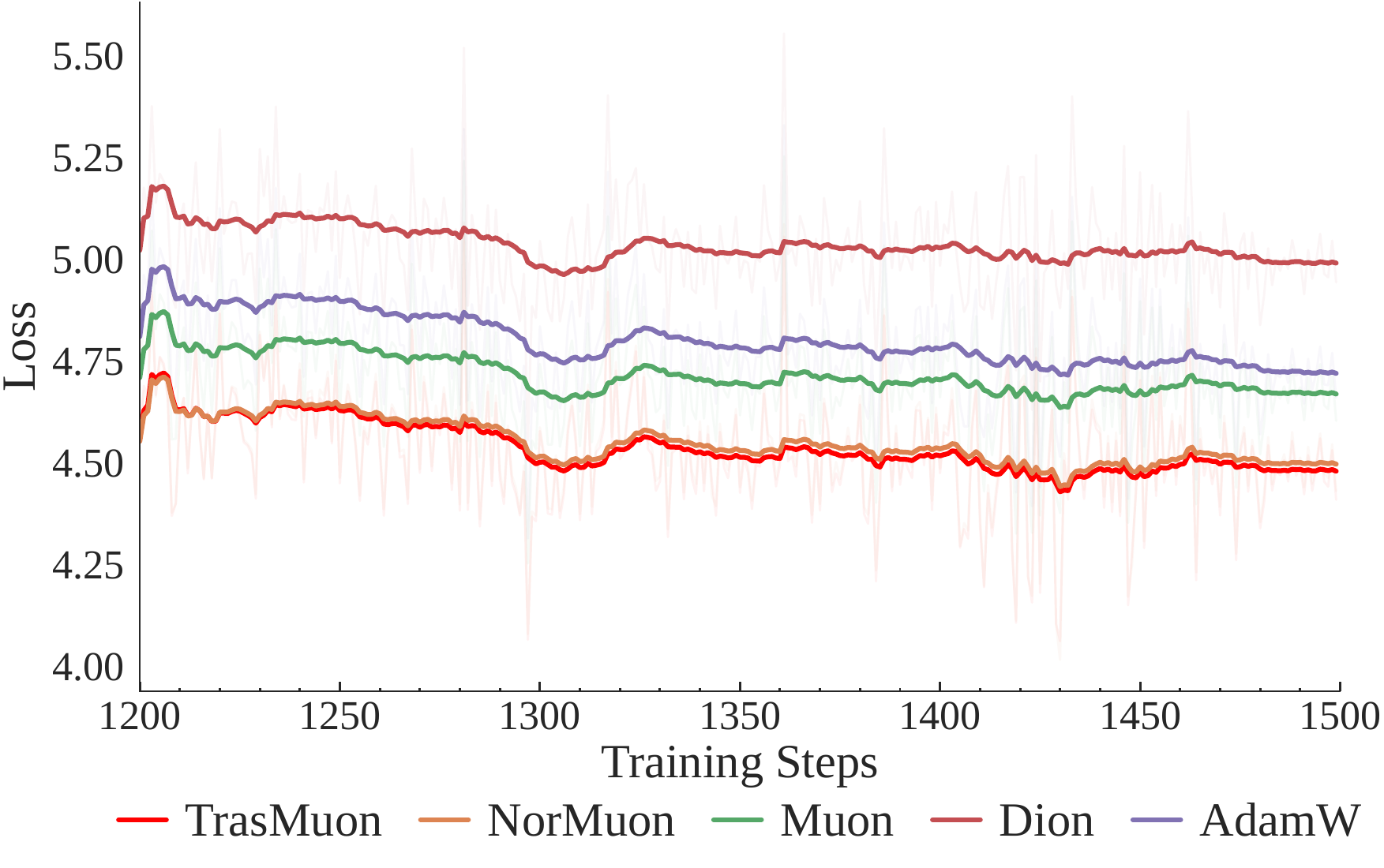}
    \caption{Warmup-enabled.}
    \label{fig:qwen_warm01_last}
  \end{subfigure}
  \hfill
  \begin{subfigure}[b]{0.49\columnwidth}
    \centering
    \includegraphics[width=\linewidth]{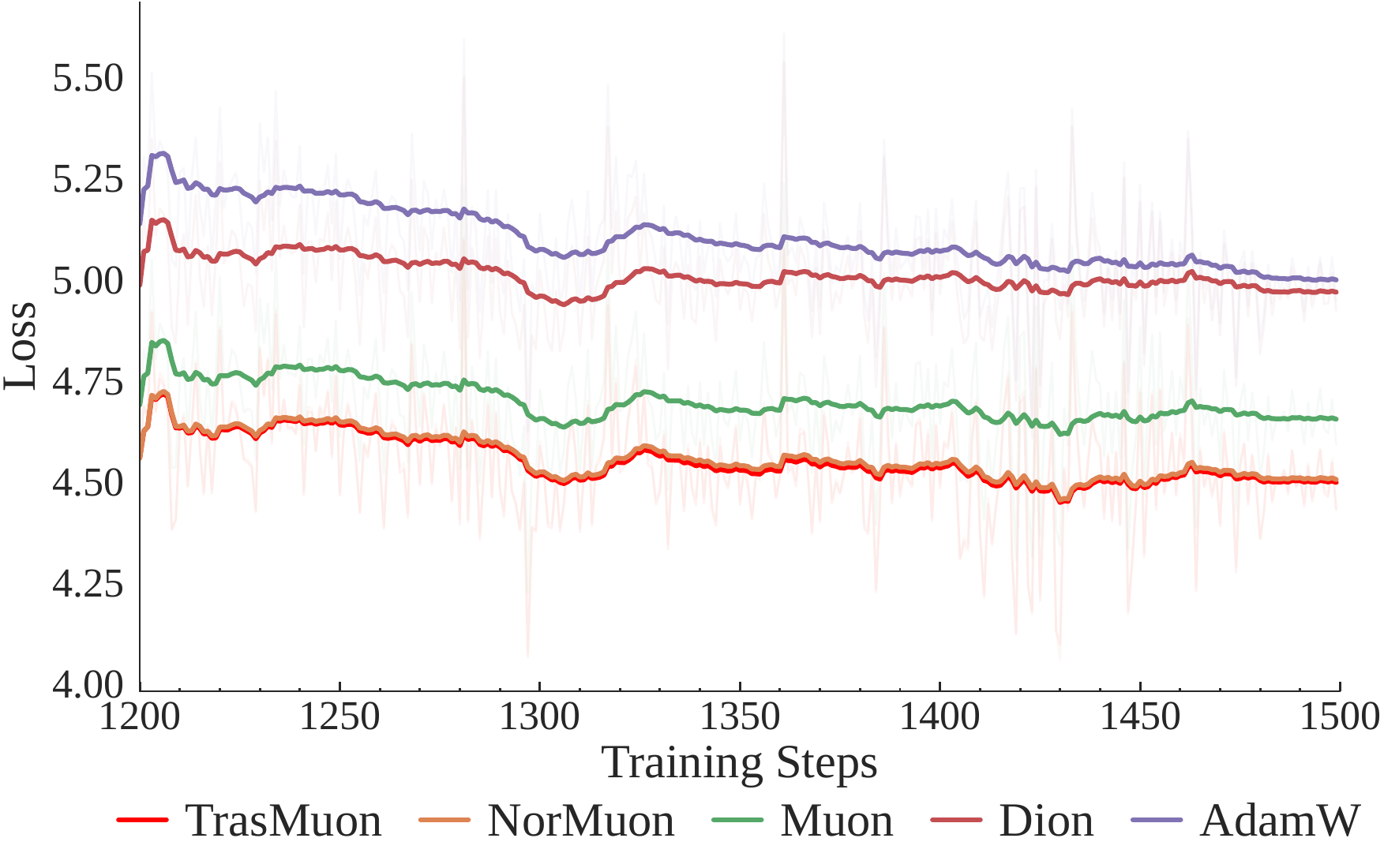}
    \caption{Warmup-free.}
    \label{fig:qwen_nowarm_last}
  \end{subfigure}
  \caption{\textbf{Qwen3-0.6B late-stage window (steps 1200--1500) under a fixed-budget protocol.}
  Training loss under warmup-enabled vs.\ warmup-free variants of the same warmup--stable--decay schedule with a final decay phase near the end of training.}
  \label{fig:qwen_last_steps}
\end{figure}

\subsection{Qwen3-0.6B: Minimal replication check}
\label{app:qwen_replication}

The main text reports Qwen3-0.6B results under a single fixed seed.
To reduce the possibility that the observed qualitative trends are specific to that particular random draw, the same Qwen3-0.6B runs are repeated with an additional random seed while keeping \emph{all} other settings identical, including the data pipeline, token budget, model architecture, learning rate, weight decay, and schedule variant.
These replication runs are not used for hyperparameter tuning and are included solely as a robustness check under an identical protocol.

Figure~\ref{fig:qwen_rep_seed42} presents early-stage loss trajectories over steps 0--350 for the additional seed.
Across both schedule variants, the qualitative optimizer behavior remains consistent with the main-text observations.
No claim of statistical significance is made from this minimal replication; the results are provided as supplementary evidence of qualitative robustness under matched conditions.

\begin{figure}[t]
  \centering
  \begin{subfigure}[b]{0.49\columnwidth}
    \centering
    \includegraphics[width=\linewidth]{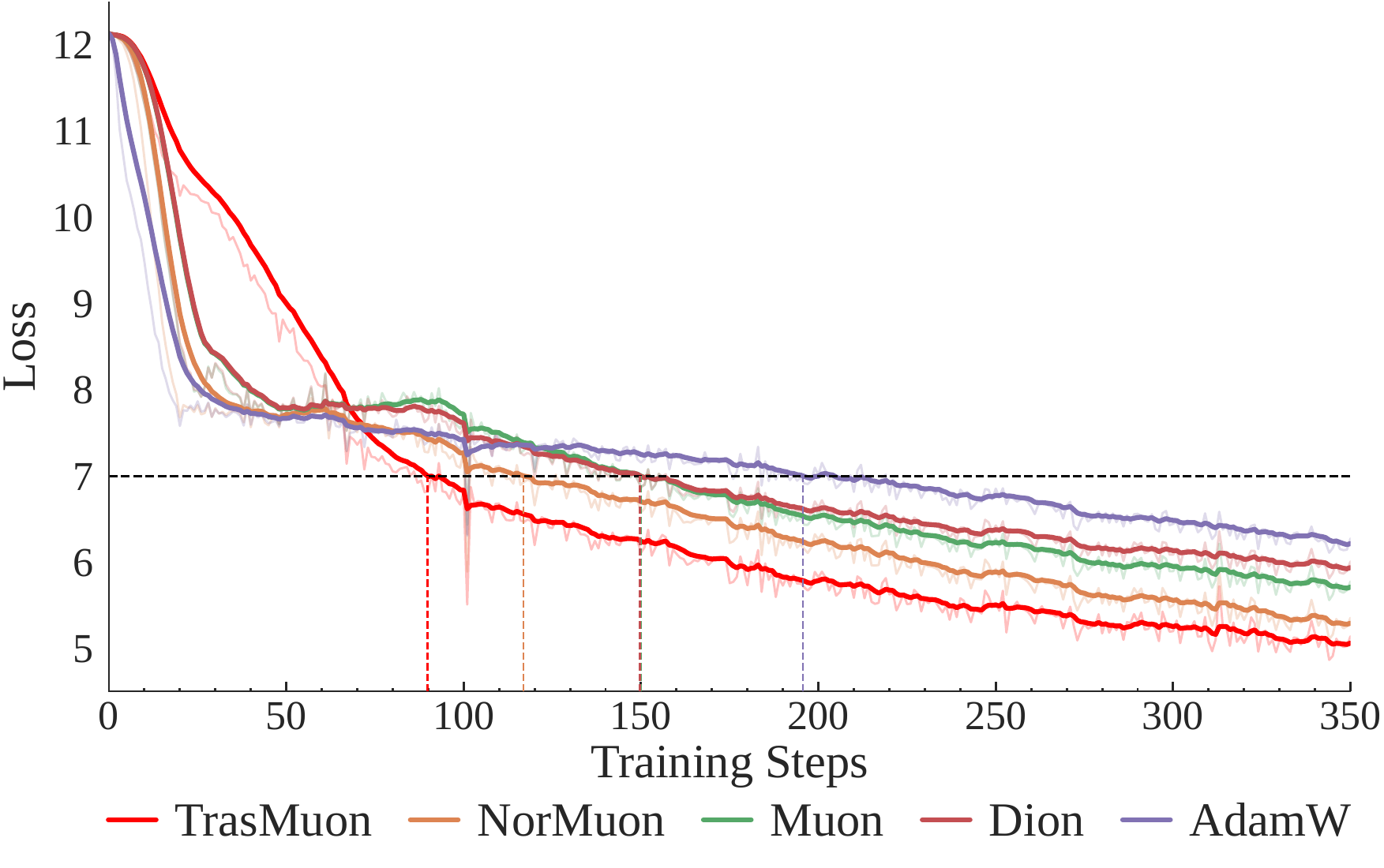}
    \caption{Warmup-enabled, seed 42.}
    \label{fig:qwen_warm01_seed42}
  \end{subfigure}
  \hfill
  \begin{subfigure}[b]{0.49\columnwidth}
    \centering
    \includegraphics[width=\linewidth]{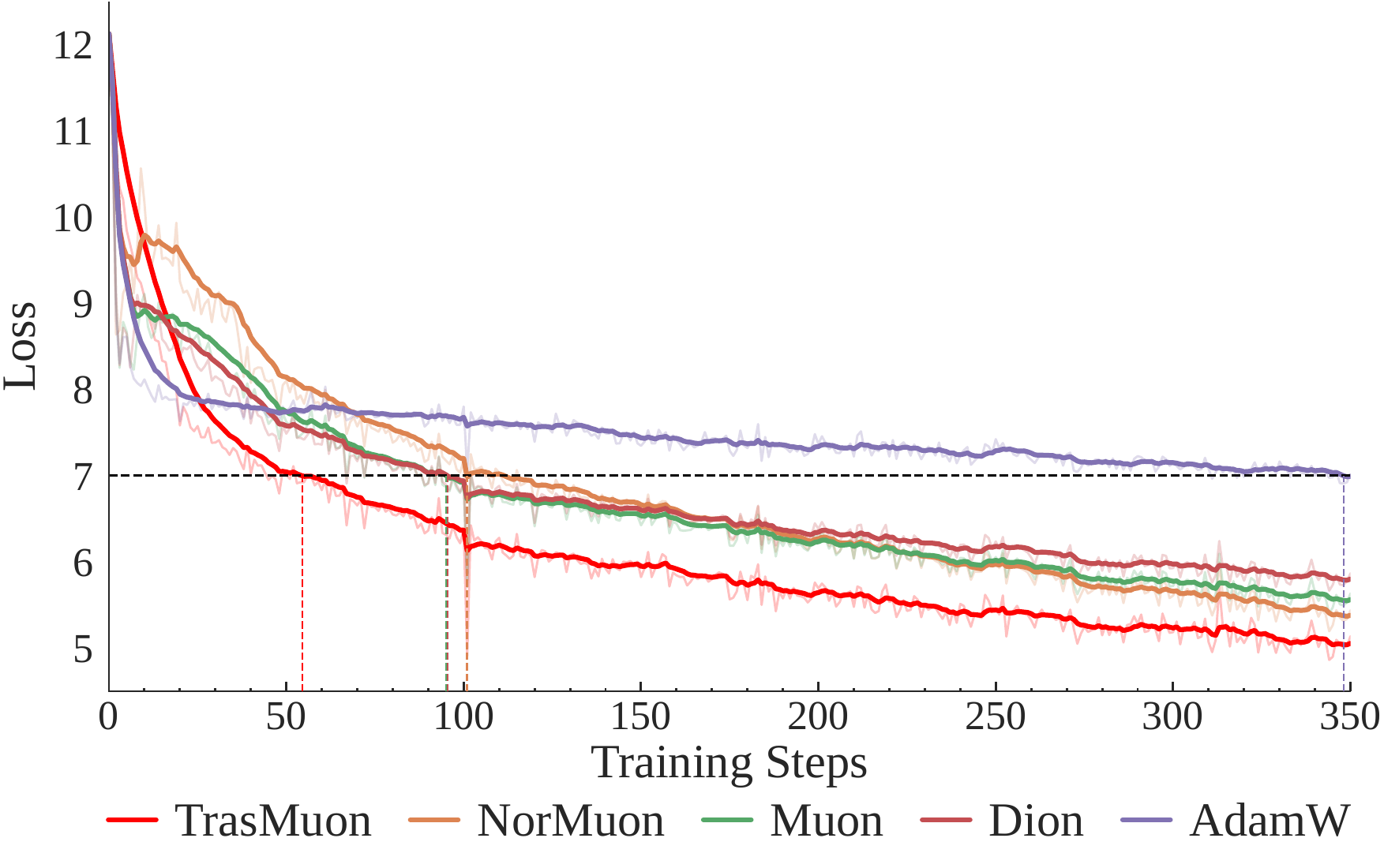}
    \caption{Warmup-free, seed 42.}
    \label{fig:qwen_nowarm_seed42}
  \end{subfigure}
  \caption{\textbf{Qwen3-0.6B replication under an identical protocol (additional seed).}
  Early-stage training loss (steps 0--350) for an additional random seed, under the same configuration as the main-text experiment.}
  \label{fig:qwen_rep_seed42}
\end{figure}

\subsection{GPT-2 Small: Additional architecture under the same protocol}
\label{app:gpt2_small}

GPT-2 Small is additionally evaluated under the same fixed-budget protocol to assess whether the observed qualitative dynamics transfer to a different architecture.
Figure~\ref{fig:gpt_small_warmup_compare} presents early-stage loss trajectories over steps 0--350 for both schedule variants, and Figure~\ref{fig:gpt_s_last_steps} reports the late-stage window over steps 1200--1500.
These plots are included to complement the main-text results, and the late-stage window is reported for completeness rather than as evidence of full convergence.

\begin{figure}[t]
  \centering
  \begin{subfigure}[b]{0.49\columnwidth}
    \centering
    \includegraphics[width=\linewidth]{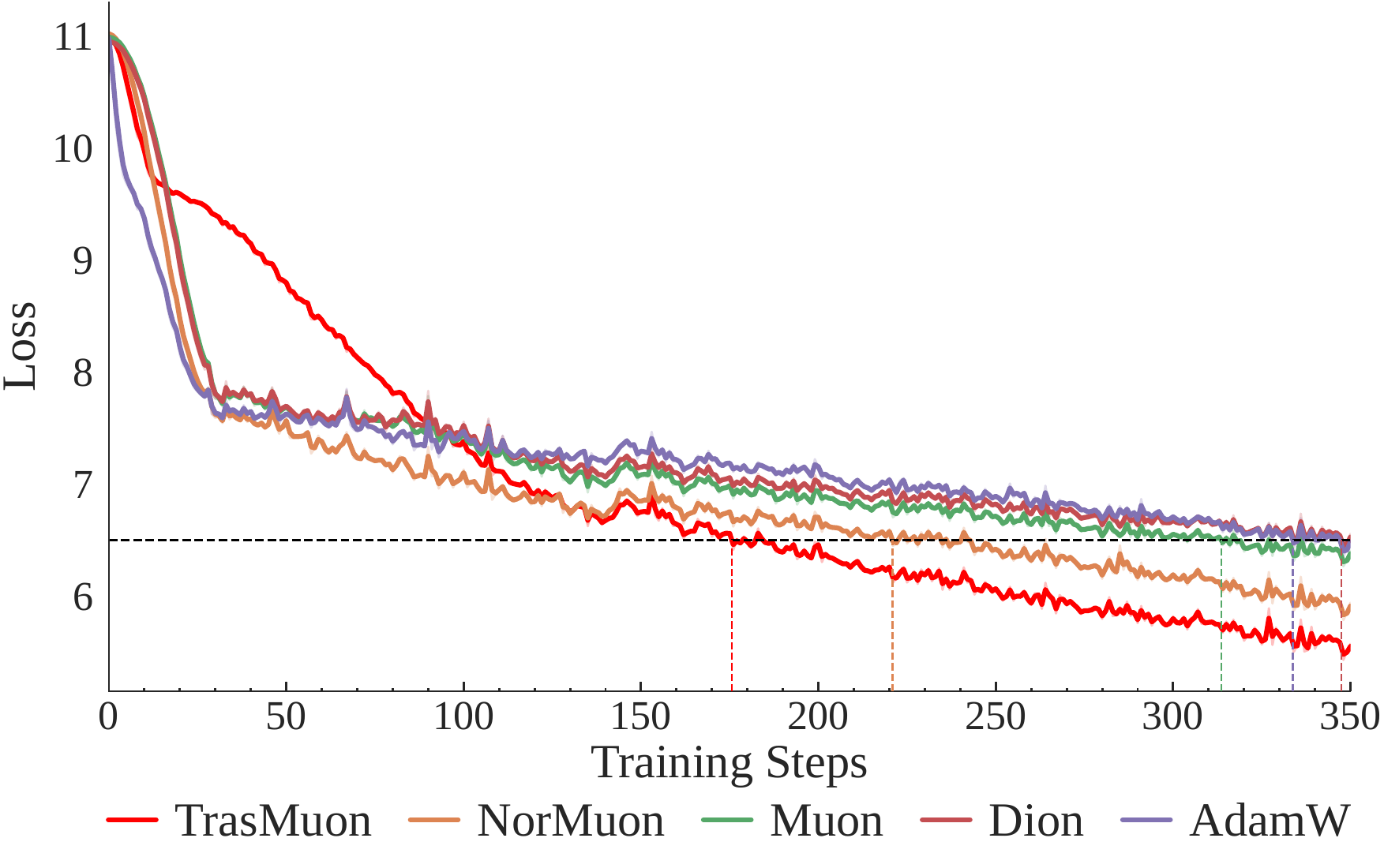}
    \caption{Warmup-enabled.}
    \label{fig:gpt_s_warm01}
  \end{subfigure}
  \hfill
  \begin{subfigure}[b]{0.49\columnwidth}
    \centering
    \includegraphics[width=\linewidth]{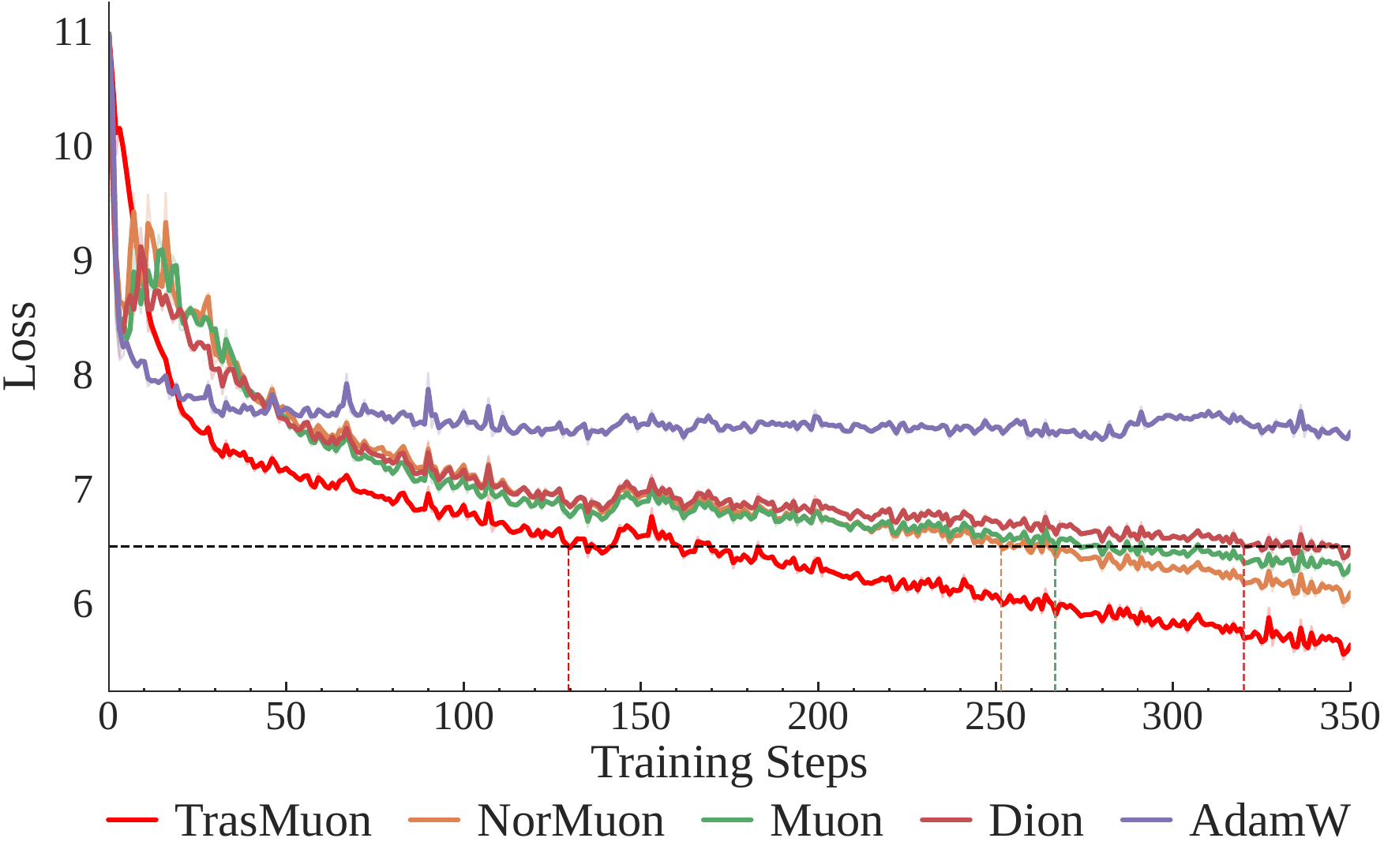}
    \caption{Warmup-free.}
    \label{fig:gpt_s_nowarm}
  \end{subfigure}
  \caption{\textbf{GPT-2 Small early-stage dynamics (steps 0--350) under a fixed-budget protocol.}
  Warmup-enabled vs.\ warmup-free variants of the same warmup--stable--decay schedule.}
  \label{fig:gpt_small_warmup_compare}
\end{figure}

\begin{figure}[t]
  \centering
  \begin{subfigure}[b]{0.49\columnwidth}
    \centering
    \includegraphics[width=\linewidth]{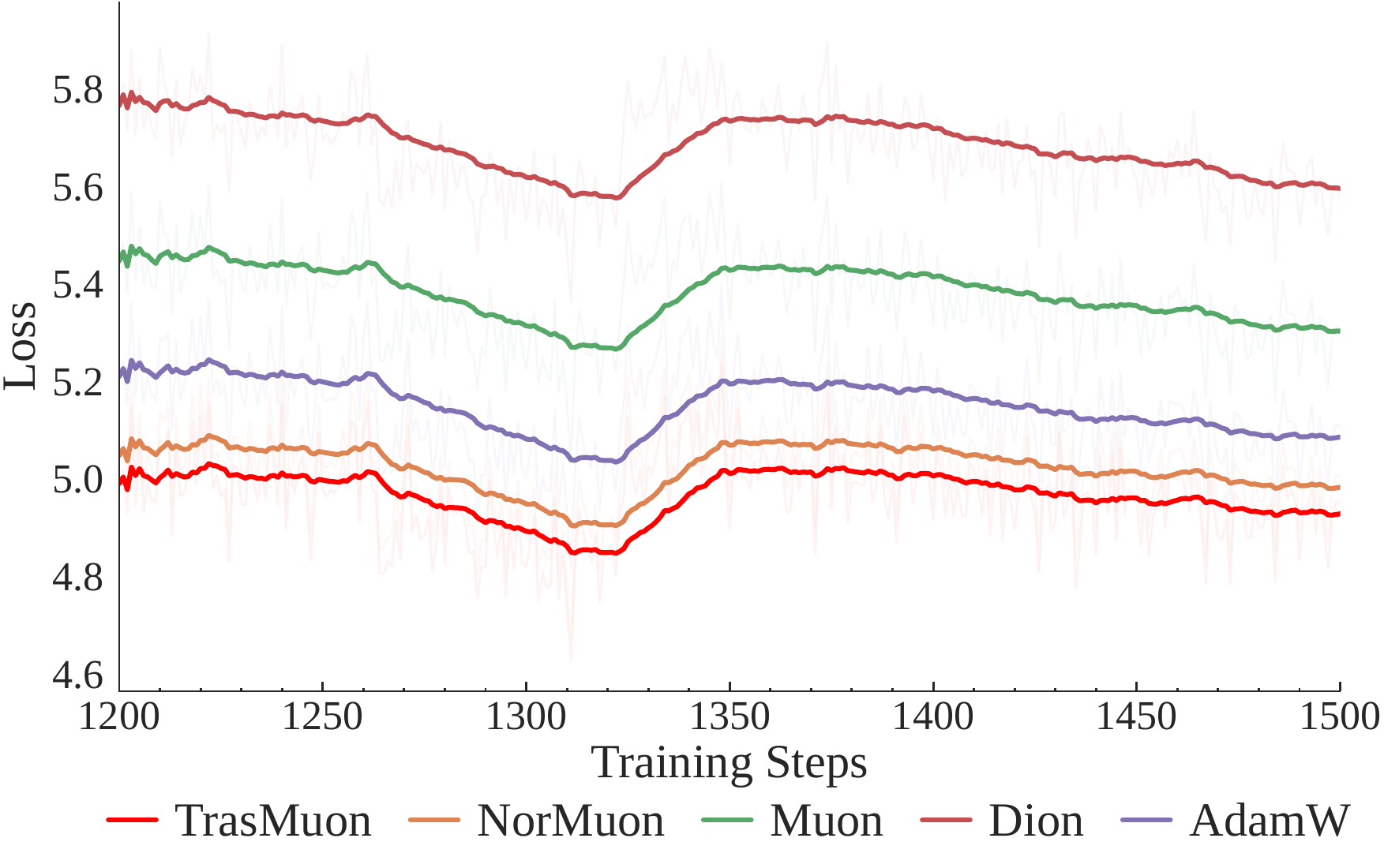}
    \caption{Warmup-enabled.}
    \label{fig:gpt_s_warm01_last}
  \end{subfigure}
  \hfill
  \begin{subfigure}[b]{0.49\columnwidth}
    \centering
    \includegraphics[width=\linewidth]{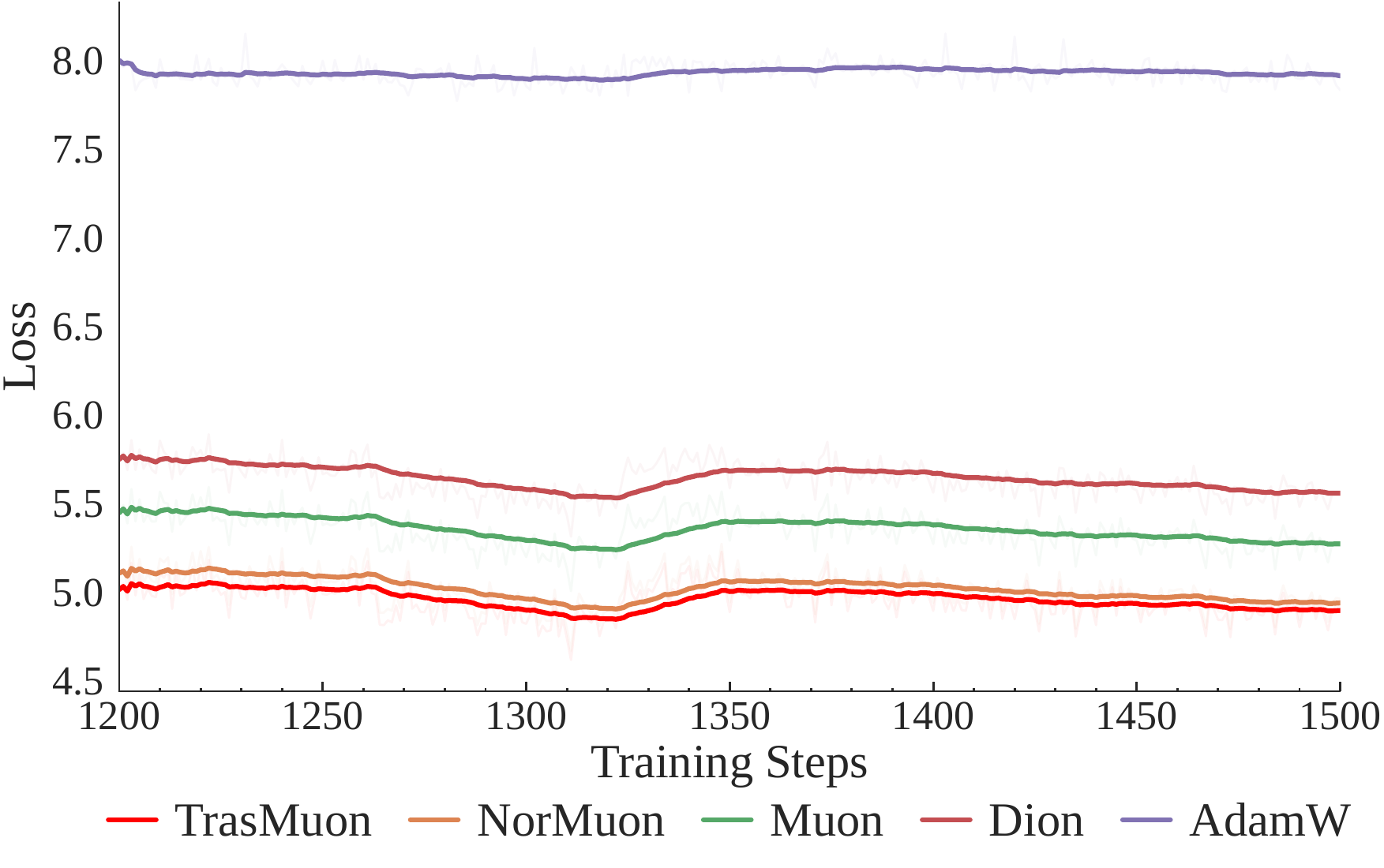}
    \caption{Warmup-free.}
    \label{fig:gpt_s_nowarm_last}
  \end{subfigure}
  \caption{\textbf{GPT-2 Small late-stage window (steps 1200--1500) under a fixed-budget protocol.}
  Training loss for warmup-enabled vs.\ warmup-free schedule variants.}
  \label{fig:gpt_s_last_steps}
\end{figure}

\section{Additional Vision Transformer Experiments}
\label{app:vit_extra}

\subsection{ImageNet-100 Data Source and Construction}
\label{app:imagenet100_data_source}
To reduce dataset preparation overhead, we use a publicly available ImageNet-100 \emph{image archive} hosted on Kaggle~\citep{ambityga2021imagenet100} as a convenient storage source.
Importantly, the Kaggle archive is used \emph{only} as a source of image files.
Class membership and the train/validation split are defined strictly according to the ILSVRC-2012 (ImageNet-1k) specification.
Concretely, we select a fixed subset of 100 ILSVRC-2012 classes (specified by synset IDs) and retain only images whose labels match these synsets; we then follow the standard ILSVRC-2012 train/validation split.
As a sanity check, we verify (in code) the synset-to-index mapping, per-class image counts, and that no images outside the selected synsets are included.
This ensures that the benchmark corresponds to a well-defined subset of ImageNet-1k, independent of the hosting platform.

\subsection{ImageNet-100 Experimental Protocol and Seed Allocation}
\label{app:imagenet100_exp_protocol}

We evaluate optimization methods on ImageNet-100 using a ViT-Base/16 architecture at $224\times224$ resolution.
Training follows a standard ViT/DeiT recipe~\citep{dosovitskiy2020vit,touvron2021training}, including random resized cropping, horizontal flipping, color jitter, RandAugment, and random erasing, together with label smoothing and Mixup (CutMix disabled).

Weight decay is applied with parameter grouping: LayerNorm and bias parameters use zero weight decay, and all remaining parameters use a fixed decay rate.
Unless otherwise stated, all experiments use a base learning rate of $1\times10^{-3}$ and a weight decay of $5\times10^{-2}$.
We compare AdamW~\citep{loshchilov2019decoupledweightdecayregularization},
Muon~\citep{jordan2024muon},
NorMuon~\citep{liNorMuonMakingMuon2025a},
and \textsc{TrasMuon} under the same model architecture, data pipeline, training schedule, and compute budget, using three random seeds (42, 43, 44).
Optimizer-specific parameters follow the respective published defaults and our unified implementation. For each method, we report the mean and standard deviation of \emph{validation} top-1 accuracy across seeds (Table~\ref{tab:vit_imagenet100}).

\begin{table}[t]
\centering
\small
\setlength{\tabcolsep}{8pt}
\renewcommand{\arraystretch}{1.15}
\caption{\textbf{ViT on ImageNet-100.}
Final \emph{validation} top-1 accuracy (mean $\pm$ standard deviation) over three random seeds (42, 43, 44).}
\label{tab:vit_imagenet100}
\begin{tabular}{lcc}
\toprule
Optimizer & Accuracy Mean & Accuracy Std \\
\midrule
\textbf{TrasMuon} & \textbf{77.47\%} & \textbf{0.34\%} \\
NorMuon & 77.10\% & 0.21\% \\
Muon & 69.69\% & 0.08\% \\
AdamW & 42.53\% & 4.38\% \\
\bottomrule
\end{tabular}
\end{table}

\subsection{CIFAR-100: Column-Energy Stress Test}
\label{app:vit_cifar_burst}

\paragraph{Setup (stress-test benchmark).}
We conduct a controlled stress test on CIFAR-100 using a Vision Transformer to assess optimizer robustness under axis-localized nonstationarity.
Unless otherwise specified, all runs use 30 epochs, batch size 128, learning rate $1\times10^{-3}$, weight decay $5\times10^{-3}$, and identical data loading and preprocessing.
We report mean and standard deviation of test accuracy over three random seeds (42, 43, 44).

\paragraph{Burst injection.}
To introduce structured nonstationarity without changing the data distribution, we inject sparse \emph{column-localized gradient bursts} into selected large 2D parameter matrices (attention and MLP projections).
The goal of this protocol is not to claim that injected bursts match the exact distribution of naturally occurring outliers, but to provide a reproducible mechanism-level stressor that concentrates energy along a small subset of feature axes.
Crucially, the same burst pattern (targeted layers, selected columns, and timesteps) is applied across optimizers by fixing the burst random seed, enabling direct, fair comparison of optimizer responses.
Full implementation details are provided in Appendix~\ref{app:burst_protocol}.

\paragraph{Observed behavior.}
Table~\ref{tab:vit_cifar100} summarizes test accuracy under burst injection.
Across this stress setting, Muon-family optimizers maintain higher accuracy and lower variability than AdamW.
Normalization-based variants reduce variance relative to Muon, while \textsc{TrasMuon} attains the highest mean accuracy and the smallest spread across seeds in this configuration.

\begin{table}[t]
\centering
\small
\setlength{\tabcolsep}{8pt}
\renewcommand{\arraystretch}{1.15}
\caption{\textbf{ViT on CIFAR-100 with column-localized gradient bursts.}
Final top-1 test accuracy (mean $\pm$ standard deviation) over three seeds (42, 43, 44).}
\label{tab:vit_cifar100}
\begin{tabular}{lcc}
\toprule
Optimizer & Accuracy Mean & Accuracy Std \\
\midrule
\textbf{TrasMuon} & \textbf{58.77\%} & \textbf{0.22\%} \\
NorMuon & 58.31\% & 0.52\% \\
Muon & 57.48\% & 0.52\% \\
AdamW & 35.03\% & 5.05\% \\
\bottomrule
\end{tabular}
\end{table}

\subsection{Burst Injection Protocol}
\label{app:burst_protocol}

We define a column-wise gradient burst operator applied to selected 2D weight matrices to induce controlled column-energy spikes without altering the data distribution.
At each burst step, we select $k$ column indices (random or fixed, as specified) and perturb each selected column by adding a normalized random direction:
\[
g_{:,j} \leftarrow g_{:,j} + \alpha \cdot \frac{u}{\lVert u\rVert_2 + \epsilon},
\quad u \sim \mathcal{N}(0, I).
\]
The amplitude $\alpha$ can be specified as a fixed absolute value or scaled relative to the current gradient magnitude using a Frobenius-normalized reference:
\[
\alpha = \rho \cdot \frac{\lVert g \rVert_F}{\sqrt{d_{\text{out}} d_{\text{in}}}},
\]
optionally clipped by a maximum threshold.
Bursts occur every $T$ optimization steps after an optional warmup phase and target only designated 2D layers (e.g., attention projections and MLP weights).
Burst events and optimizer internal statistics (including feature-wise clipping coefficients, when applicable) are logged at the same timesteps to enable direct alignment between perturbations and optimizer responses.

\section{Additional PINN Details: Random-ROI Sampling Stress Test}
\label{app:pinn_roi_protocol}

\paragraph{Task recap.}
We consider the Helmholtz equation on $\Omega=[0,1]^2$ with homogeneous Dirichlet boundary conditions and a manufactured solution $u^\star$ (see Section~\ref{sec:pinns} for the PDE, loss definition, and the rel$L_2$ metric).
This appendix specifies the non-stationary sampling protocol used to stress-test optimizer robustness under controlled distribution shifts.

\paragraph{Training configuration.}
We use an MLP with hidden dimensions $(128,128,128)$ and $\tanh$ activation, trained for $4000$ steps.
At each step we sample $N_r=1024$ interior points and $N_b=256$ boundary points with boundary weight $\lambda_b=100$.
We evaluate every $200$ steps on a fixed $128\times128$ grid.

\paragraph{ROI sampling as a distribution shift.}
Starting from step $t_0=1000$, we trigger ROI events every $K_{\mathrm{out}}=20$ steps.
At an ROI event step $t$, we sample interior points from a mixture distribution
\begin{equation}
p_t(x) \;=\; (1-\alpha)\,p_0(x) + \alpha\,p_{\mathrm{roi}}^{(t)}(x),
\qquad \alpha = 0.95,
\label{eq:roi_mixture}
\end{equation}
where $p_0$ is uniform over $\Omega$ and $p_{\mathrm{roi}}^{(t)}$ is uniform over the selected ROI patch $\Omega_{\mathrm{roi}}^{(t)}$.
This produces a controlled, non-stationary sampling distribution (a stressor analogous to adaptive/ROI refinement in practical PINNs).

\paragraph{Random ROI patch pool (reproducible).}
The ROI patch $\Omega_{\mathrm{roi}}^{(t)}$ is selected from a fixed pool (Table~\ref{tab:roi_patch_pool}) using a deterministic \texttt{step\_hash} seeding rule.
Thus ROI locations vary across events yet remain fully reproducible given the experiment configuration and training step index.
\begin{table}[t]
\centering
\small
\caption{ROI patch pool used for random ROI events (rectangles are $[x_0,x_1]\times[y_0,y_1]$).}
\label{tab:roi_patch_pool}
\begin{tabular}{l}
\toprule
ROI patches \\
\midrule
Corners: $[0.00,0.03]\times[0.00,0.03]$, $[0.97,1.00]\times[0.00,0.03]$, $[0.00,0.03]\times[0.97,1.00]$, $[0.97,1.00]\times[0.97,1.00]$ \\
Edges: $[0.48,0.53]\times[0.00,0.05]$, $[0.48,0.53]\times[0.95,1.00]$, $[0.00,0.05]\times[0.48,0.53]$, $[0.95,1.00]\times[0.48,0.53]$ \\
Interior: $[0.20,0.25]\times[0.20,0.25]$, $[0.45,0.50]\times[0.10,0.15]$, $[0.10,0.15]\times[0.55,0.60]$, $[0.60,0.65]\times[0.60,0.65]$ \\
\bottomrule
\end{tabular}
\end{table}

\paragraph{ROI-local evaluation.}
To quantify local disturbance and recovery, at ROI event steps we additionally compute an ROI-local rel$L_2$ on a $64\times64$ grid restricted to $\Omega_{\mathrm{roi}}^{(t)}$, and estimate non-ROI error by sampling $4096$ points from $\Omega\setminus \Omega_{\mathrm{roi}}^{(t)}$.
These diagnostics separate global convergence from localized behavior under distribution shifts.

\subsection{Step-Size Alignment and Learning-Rate Sensitivity (PINN Helmholtz, $k=2$)}
\label{app:pinn_stepsize_lr}

\paragraph{Motivation.}
Nominal learning rates are not directly comparable across update rules because different optimizers can induce different \emph{effective} parameter-space step magnitudes.
To reduce the confound that performance differences are driven by trivial step-size mismatches, we complement the main comparison with (i) a short step-size alignment diagnostic and (ii) a shared learning-rate (LR) sweep.

\paragraph{Step-size alignment diagnostic.}
Figure~\ref{fig:pinn_stepsize_lr}a reports the achieved effective step size,
computed from parameter differences during an initial stationary window (before any ROI perturbations are introduced),
\[
s_t \;=\; \frac{\|\Delta\theta_t\|_2}{\sqrt{P}}, \qquad P=\mathrm{dim}(\theta).
\]
We target a fixed reference magnitude (dashed line) with a tolerance band (shaded region).
Muon and \textsc{TrasMuon} attain comparable achieved step sizes within the tolerance range, supporting that subsequent robustness comparisons are not explained by a simple global step-size discrepancy.

\paragraph{Learning-rate sweep.}
Figure~\ref{fig:pinn_stepsize_lr}b shows the final relative $L_2$ error under a shared LR sweep.
Shaded bands summarize variability across random seeds (median with an interquartile range).
Both methods exhibit the expected degradation as LR increases beyond the stable region.
Together with the alignment diagnostic, this sweep provides a complementary view of optimizer sensitivity to step-size calibration under the same training and sampling protocol.

\begin{figure}[t]
  \centering
  \begin{subfigure}[b]{0.49\columnwidth}
    \centering
    \includegraphics[width=\linewidth]{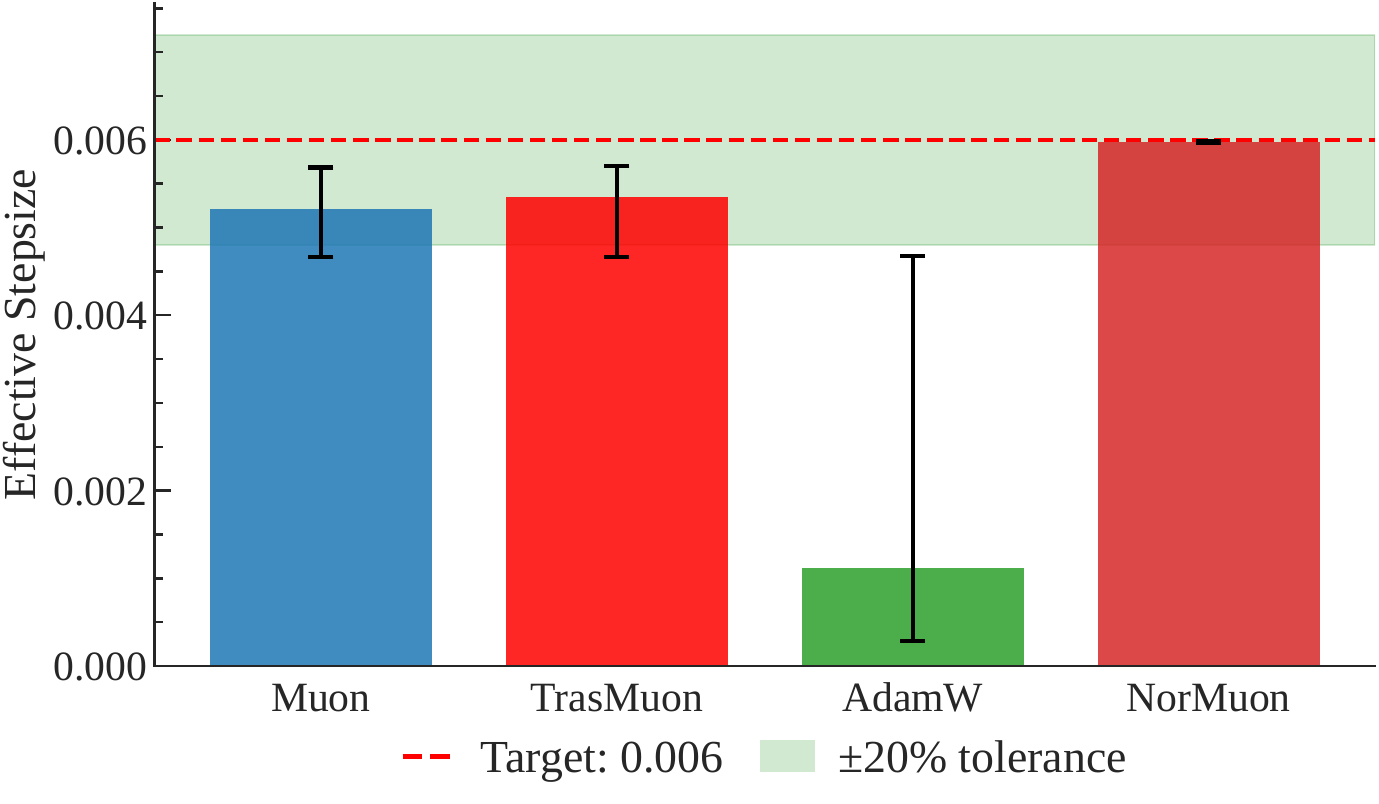}
    \caption{\textbf{Achieved effective step size} during the stationary alignment window (before ROI events). The dashed line is the target, and the shaded region indicates tolerance. Error bars denote variability across seeds.}
    \label{fig:pinn_stepsize_align}
  \end{subfigure}\hfill
  \begin{subfigure}[b]{0.49\columnwidth}
    \centering
    \includegraphics[width=\linewidth]{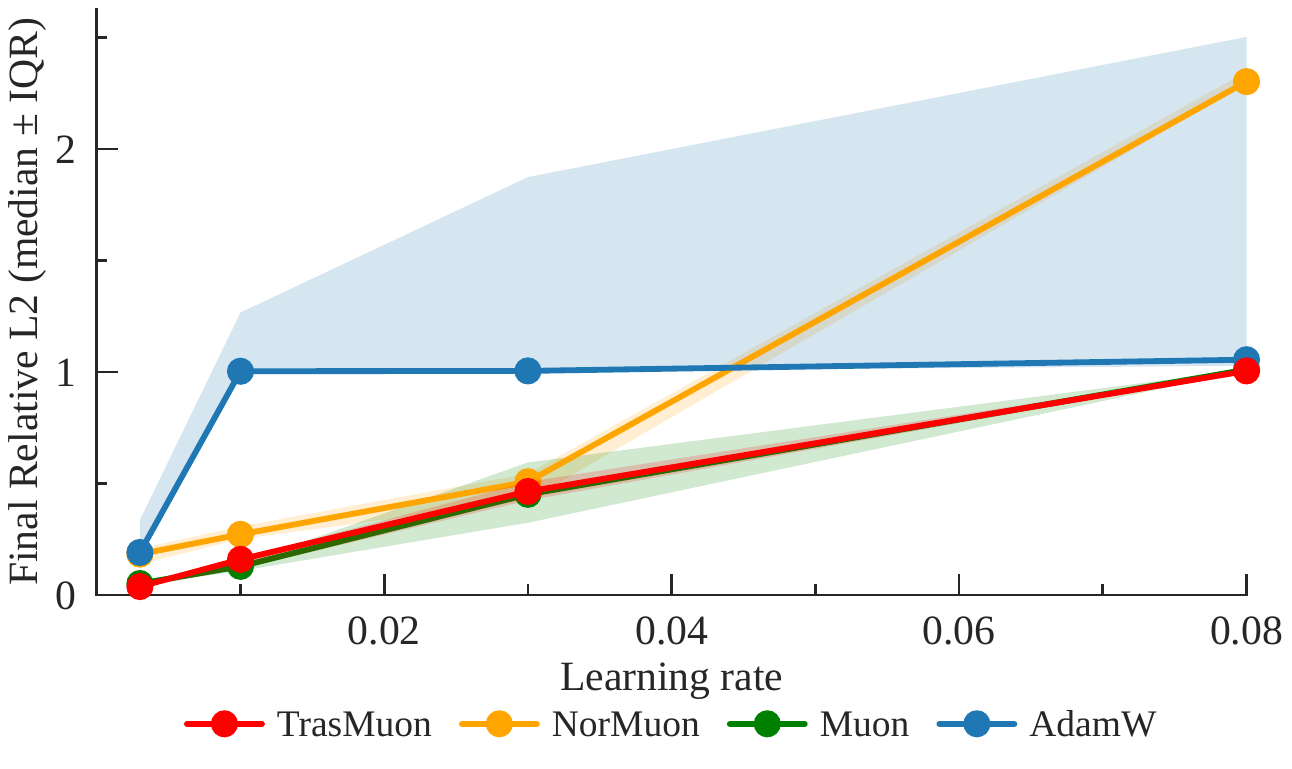}
    \caption{\textbf{LR sensitivity} of final relative $L_2$ error under a shared LR sweep. Lines show the median across seeds and shaded bands indicate the IQR.}
    \label{fig:pinn_lr_sweep}
  \end{subfigure}
  \caption{\textbf{PINN Helmholtz ($k=2$): step-size alignment and LR sensitivity.}
  (a) Effective step-size alignment reduces trivial magnitude confounds when comparing optimizers with different update rules.
  (b) A shared LR sweep summarizes sensitivity to step-size calibration via final relative $L_2$ error.}
  \label{fig:pinn_stepsize_lr}
\end{figure}

\subsection{Additional PINN Diagnostics: Metric Distributions Across Seeds}
\label{app:pinn_metric_dist}

We visualize the distribution of key robustness and accuracy metrics across random seeds for the PINN Helmholtz benchmark(\(k=2\)). This figure serves as a distributional check to ensure that the reported trends are not driven by a single favorable run.

\begin{figure}[t]
  \centering
  \includegraphics[width=0.8\columnwidth] {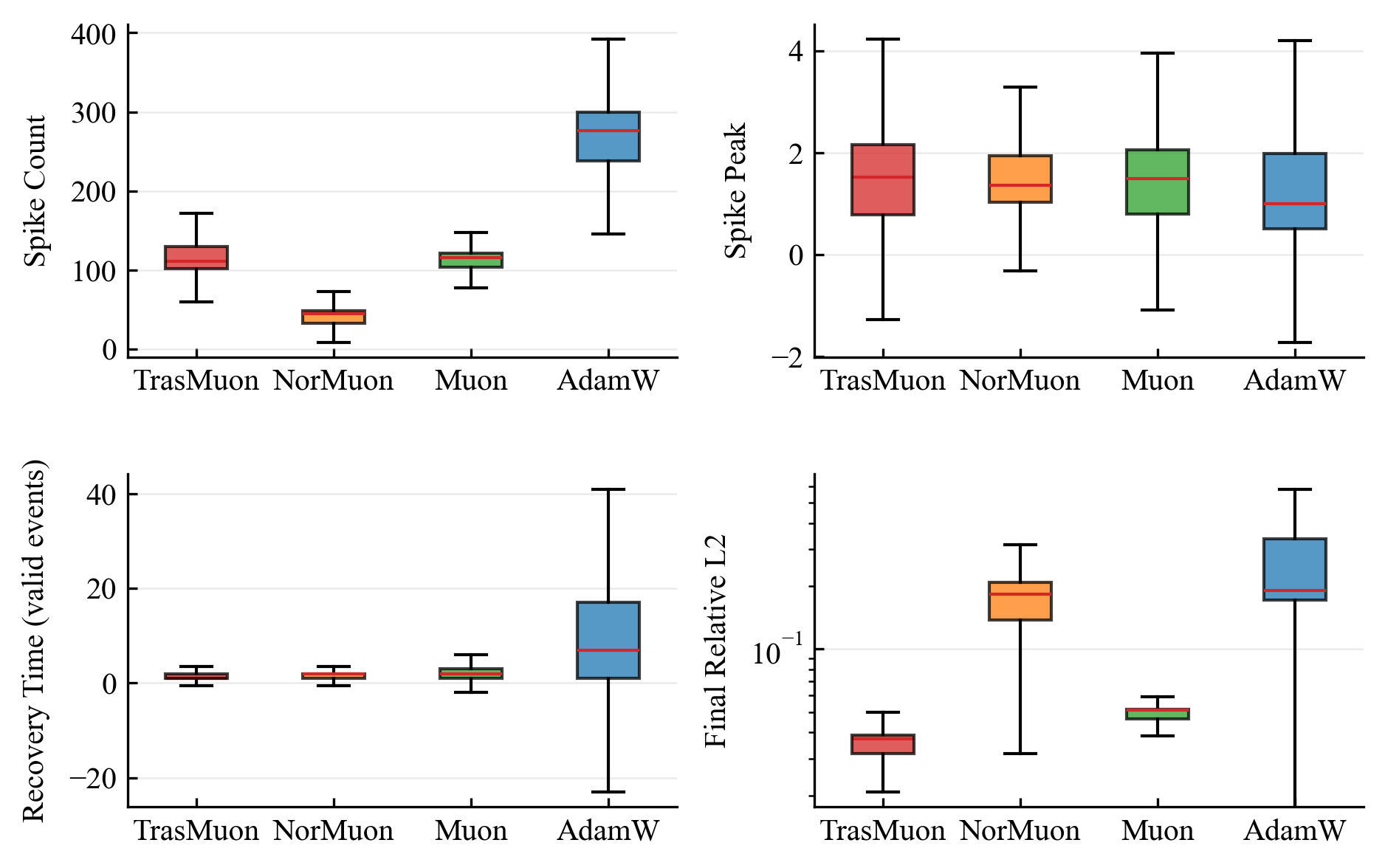}
  \caption{\textbf{PINN Helmholtz(\(k=2\)): metric distributions across seeds.}
  Boxplots summarize spike count, spike peak, recovery time (valid events), valid-event rate, and final relative \(L_2\) error across seeds under the same training protocol.
  Each box shows the median and interquartile range (IQR); whiskers indicate the remaining spread.}
  \label{fig:pinn_k2_metric_distributions}
\end{figure}

\section{Supplementary Details for the Mechanistic Study}
\label{app:toy2_supp}

This appendix provides protocol-level details and supporting evidence for the mechanistic study in Section~\ref{sec:toy2}.
The intent is two-fold: (i) to make the stress protocol fully reproducible, and (ii) to document a minimal, time-aligned evidence chain
that is \emph{consistent with} the intended energy-indexed, feature-wise clipping mechanism under a controlled intervention.
We do not introduce new claims beyond Section~\ref{sec:toy2}.

\subsection{Reproducible protocol}
\label{app:toy2_protocol}

\paragraph{Quadratic objective and stiffness.}
We optimize a matrix parameter $W\in\mathbb{R}^{d\times d}$ under
\begin{equation}
f(W) \;=\; \tfrac12 \|A W B - T\|_F^2,
\end{equation}
where $A=U\Sigma_AU^\top$ and $B=V\Sigma_BV^\top$.
Here $U,V$ are orthogonal matrices and $\Sigma_A,\Sigma_B$ are diagonal spectra constructed to achieve a target condition number
$\kappa\in\{10^2,10^4,10^6\}$.
We repeat each configuration over multiple random seeds/rotations and report robust statistics (median and IQR).

\paragraph{Column-localized burst injection (controlled intervention).}
Every $K_{\mathrm{out}}$ steps we inject an outlier event by amplifying a small subset of momentum columns.
Concretely, we sample $\mathcal{J}\subseteq[d]$ with $|\mathcal{J}|=s\ll d$ and apply
\begin{equation}
\widetilde{M}_{t,\cdot j} \;=\;
\begin{cases}
a\,M_{t,\cdot j}, & \text{if } j\in\mathcal{J},\\[2pt]
M_{t,\cdot j}, & \text{otherwise},
\end{cases}
\qquad a>1,
\label{eq:toy2_burst_app}
\end{equation}
then form the optimizer update at that step using $\widetilde{M}_t$ in place of $M_t$.
This intervention changes the \emph{optimization dynamics} but keeps the underlying objective $f(W)$ unchanged.

\paragraph{Preserving vs.\ breaking feature/column semantics (\texttt{fix\_V}).}
Feature-wise clipping is axis-aligned (column-wise), so its effect depends on whether the chosen column axis remains meaningful.
In the main setting (\texttt{fix\_V=True}), we preserve the column basis used for burst injection, so ``columns'' correspond to consistent feature axes across training and injected events.
As a boundary condition (\texttt{fix\_V=False}), we apply an additional orthogonal mixing in column space, so injected energy is dispersed across columns; in this case, selective column-wise clipping is not expected to provide the same advantage.

\subsection{Minimal evidence chain (time-aligned observables)}
\label{app:toy2_chain}

Section~\ref{sec:toy2} argues that spike suppression is consistent with the following ordering under a controlled intervention:
\emph{outlier injection} $\rightarrow$ \emph{relative energy increases} $\rightarrow$ \emph{stronger applied clipping} $\rightarrow$ \emph{reduced loss spikes}.
We summarize the corresponding observables, which are directly logged and visualized in Fig.~\ref{fig:toy2_closed_loop}.

\paragraph{(1) Outlier injection increases relative energy.}
Eq.~\eqref{eq:toy2_burst_app} increases the column energies
$E_{t,j}=\sum_i \widetilde{M}_{t,ij}^2$ for $j\in\mathcal{J}$.
We track the relative energy ratio
\begin{equation}
r_{t,j}\;=\;\frac{E_{t,j}}{E_t^{\mathrm{ref}}+\epsilon},
\end{equation}
where $E_t^{\mathrm{ref}}$ is the running reference used by \textsc{TrasMuon}.
We visualize robust summaries such as $r_{q95}$ and $r_{\max}$, which rise at injected outlier steps (Fig.~\ref{fig:toy2_closed_loop}, top).

\paragraph{(2) Higher relative energy is followed by stronger applied clipping.}
\textsc{TrasMuon} produces damping-only clipping coefficients $c_{t,j}^{\mathrm{used}}\in[c_{\min},1]$ that decrease with $r_{t,j}$.
Consistent with this design, outlier steps are followed by a decrease in the applied signal, visible via
$c_{\mathrm{used,min}}=\min_j c_{t,j}^{\mathrm{used}}$ (Fig.~\ref{fig:toy2_closed_loop}, bottom).
This time alignment is consistent with the intended ordering ``energy rise $\rightarrow$ clipping increase'' under the intervention.

\paragraph{(3) Applied clipping attenuates column updates (selectively).}
Given the matrix-form update,
\begin{equation}
\Delta W_t \;=\; -\,\hat{\eta}_t\,O_t^{\mathrm{base}}\,\mathrm{diag}(c_t^{\mathrm{used}}),
\end{equation}
each column update magnitude is scaled by $c_{t,j}^{\mathrm{used}}$:
\begin{equation}
\|\Delta W_{t,\cdot j}\|_2
\;=\;
c_{t,j}^{\mathrm{used}}\;\hat{\eta}_t\;\|O_{t,\cdot j}^{\mathrm{base}}\|_2
\;\le\;
\hat{\eta}_t\;\|O_{t,\cdot j}^{\mathrm{base}}\|_2.
\label{eq:toy2_col_atten}
\end{equation}
Thus clipping is \emph{selective}: only columns with $c_{t,j}^{\mathrm{used}}<1$ are damped, while the structured direction $O_t^{\mathrm{base}}$ is preserved.

\paragraph{(4) Spike suppression and objective improvement.}
Consistent with (1)--(3), \textsc{TrasMuon} reduces loss spikes around outlier events (Fig.~\ref{fig:toy2_loss_zoom}) and achieves lower final objective values than the backbone under matched compute (Table~\ref{tab:toy2_fixV_true}).
Spike metrics (count/peak) are computed using the same deterministic detection rule across methods; details are provided in our experiment scripts and plotting utilities.

\paragraph{Controls that break the chain.}
We include two controls that remove key requirements of the mechanism:
(i) \textsc{TrasMuon-noClip} sets $c_t^{\mathrm{used}}\equiv \mathbf{1}$, removing the attenuation in Eq.~\eqref{eq:toy2_col_atten}; empirically it behaves similarly to NorMuon (Table~\ref{tab:toy2_fixV_true}).
(ii) Under \texttt{fix\_V=False}, the injected energy is dispersed across columns, so clipping is no longer aligned with injected directions; correspondingly, the advantage of feature-wise clipping diminishes (Table~\ref{tab:toy2_fixv_false}).




\begin{table}[htb]
\centering
\caption{\textbf{Toy-2 boundary condition (\texttt{fix\_V=False}).}
When feature/column semantics are broken by column-space mixing, the advantage of feature-wise clipping diminishes.}
\label{tab:toy2_fixv_false}
\small
\begin{tabular}{lcc}
\toprule
Method & Spike Count & Final Loss \\
\midrule
NorMuon & 79 \tiny{(74,86)} & \textbf{1.1e+06} \tiny{(8.4e+05,1.6e+06)} \\
\textsc{TrasMuon}-noClip & 80 \tiny{(74,87)} & 1.4e+06 \tiny{(9.8e+05,1.7e+06)} \\
\textsc{TrasMuon}-clip-only & 74 \tiny{(67,80)} & 1.5e+06 \tiny{(1.2e+06,1.8e+06)} \\
\textsc{TrasMuon}-clip+SF & \textbf{72} \tiny{(65,80)} & 1.3e+06 \tiny{(9.4e+05,1.7e+06)} \\
\bottomrule
\end{tabular}
\end{table}

\end{document}